\documentclass{article}
\usepackage{ijcai26}

\usepackage{times}
\usepackage{soul}
\usepackage{url}
\usepackage[title]{appendix}
\usepackage[hidelinks]{hyperref}
\usepackage[utf8]{inputenc}
\usepackage[small]{caption}
\usepackage{graphicx}
\usepackage{amsmath}
\usepackage{amsthm}
\usepackage{amssymb}
\usepackage{amsfonts}
\usepackage{booktabs}
\usepackage{algorithm}
\usepackage{algorithmic}

\usepackage[switch]{lineno}

\usepackage{mathtools}
\usepackage{xspace}
\newcommand{\R}{\mathbb{R}}

\newcommand{\N}{\mathbb{N}}
\newcommand{\Z}{\mathbb{Z}}

\newcommand\Distr{\textit{Distr}}

\newcommand\mdp{\mathcal{M}}
\newcommand\states{S}
\newcommand\actions{A}
\newcommand\transitions{\delta}

\newcommand\dtmc{\mathcal{D}}

\newcommand{\paths}{\textit{Paths}}

\newcommand\mdptup[1] [\refl]{\langle #1[\states], \allowbreak #1[\actions], \allowbreak #1[\sinit], \allowbreak #1[\transitions] \rangle}

\newcommand{\state}{s}
\newcommand{\sinit}{\state_0}
\newcommand\action{a}

\newcommand{\strategy}{\sigma}
\newcommand{\strategies}{\Sigma}

\DeclareMathOperator*{\U}{\textbf{U}}

\newcommand*{\cond}[1]{s_0, #1, t}
\newcommand*{\condPath}[1]{#1, t}

\DeclareMathOperator{\imp}{imp}

\newcommand{\impSigma}[1]{\imp_{\mdp^{\strategy}}^{s_0, t}(#1)}

\DeclareMathOperator{\Reach}{\overleftarrow{R}}

\newcommand{\qp}{$\text{QP}^*$\xspace}
\newcommand{\lp}{$\text{LP}^*$\xspace}
\newcommand{\gqp}{QP\xspace}

\usepackage{tikz}
\usetikzlibrary{positioning,angles,quotes}
\usetikzlibrary{arrows.meta}
\usetikzlibrary{shapes,snakes}
\usetikzlibrary{calc}
\usetikzlibrary{
  decorations.markings,
  decorations.pathmorphing,
  decorations.text,
}

\usepackage{diagbox}
\usepackage{todonotes}
\usepackage{mdframed}
\newenvironment{problemstatement}[1]
  {\mdfsetup{
    frametitle={\colorbox{white}{\space#1\space}},
    innertopmargin=-5pt,
    frametitleaboveskip=-\ht\strutbox,
    frametitlealignment=\center
    }
  \begin{mdframed}[nobreak=true]%
  }
  {\end{mdframed}}

\usepackage[draft,silent]{fixme}
\fxsetup{theme=color,mode=multiuser}
\definecolor{fxtarget}{rgb}{0.8000,0.0000,0.0000}
\FXRegisterAuthor{ap}{AP}{\color{orange} Andrea}
\FXRegisterAuthor{pk}{PK}{Paul}

\usepackage{subcaption}

\usepackage{xcolor}

\usepackage{multirow}

\newcommand{\unit}[1]{\,\mathrm{#1}}

\newtheorem{example}{Example}
\newtheorem{theorem}{Theorem}
\newtheorem{lemma}{Lemma}
\newenvironment{proofsketch}{%
\proof}{\endproof}

\usepackage{cleveref}
\crefname {algorithm}   {Alg.}        {algorithms}
\Crefname {algorithm}   {Algorithm}   {Algorithm}
\crefname {equation}    {Eq.}         {equations}
\Crefname {equation}    {Equation}    {Equations}
\crefname {proposition} {Prop.}       {propositions}
\Crefname {proposition} {Proposition} {Propositions}
\crefname {lemma}       {Lemma}       {lemmata}
\Crefname {lemma}       {Lemma}       {Lemmata}
\crefname {listing}     {Listing}     {listings}
\Crefname {listing}     {Listing}     {Listings}
\crefname {definition}  {Def.}        {definitions}
\Crefname {definition}  {Definition}  {Definitions}
\crefname {theorem}     {Thm.}        {theorems}
\Crefname {theorem}     {Theorem}     {Theorems}
\crefname {figure}      {Fig.}        {figures}
\Crefname {figure}      {Figure}      {Figures}
\crefname {figure*}     {Fig.}        {figures}
\Crefname {figure*}     {Figure}      {Figures}
\crefname {page}        {p.}          {pages}
\Crefname {page}        {Page}        {Pages}
\crefname {section}     {Sect.}       {sections}
\Crefname {section}     {Section}     {Sections}
\crefname {example*}    {Example}     {examples}
\Crefname {example*}    {Example}     {Examples}
\crefname {table}       {Tbl.}        {tables}
\Crefname {table}       {Table}       {Tables}

\title{Attribution-based Explanations for Markov Decision Processes}

\author{
Paul Kobialka$^1$
\and
Andrea Pferscher$^1$\and
Francesco Leofante$^2$\and
Erika Ábrahám$^3$\and\\
Silvia Lizeth Tapia Tarifa$^1$\And
Einar Broch Johnsen$^1$\\
\affiliations
$^1$University of Oslo, Norway\\
$^2$Imperial College London, United Kingdom\\
$^3$RWTH Aachen University, Germany\\
\emails
\{paulkob, sltarifa, einarj\}@ifi.uio.no,
f.leofante@imperial.ac.uk,\\
abraham@cs.rwth-aachen.de
}

\begin{document}

\maketitle

\begin{abstract}
  Attribution techniques explain the outcome of an AI model by assigning a numerical score to its inputs. So far, these techniques have mainly focused on attributing importance to static input features at a single point in time, and thus fail to generalize to sequential decision-making settings. This paper fills this gap by introducing techniques to generate attribution-based explanations for Markov Decision Processes (MDPs). We give a formal characterization of what attributions should represent in MDPs, focusing on explanations that assign importance scores to both individual states and execution paths. We  show
how importance scores can be computed  by leveraging techniques for strategy synthesis, enabling the efficient computation of these scores despite the non-determinism inherent in an MDP. We evaluate our approach on five case-studies, demonstrating its utility in providing interpretable insights into the logic of sequential decision-making agents.

\end{abstract}

\section{Introduction}
\label{sec:introduction}
Let us consider a loan application process, in which clients applying for a loan interact with a bank to discuss their eligibility and receive quotes on possible loans.  While established techniques from Process Mining and Machine Learning can be used to automate process monitoring and decision-making (see, e.g.,~\cite{teinemaa2019outcome,leo2019machine,shi2022machine,montevechi2024advancing}), the overall procedure remains inherently semi-structured. A successful application involves a sequence of interdependent steps, e.g., consultations and document submissions. Each step influences the probability of a successful outcome. Because the impact of these steps is often uncertain and state-dependent, the process can be viewed as a series of state transitions guided by both agent decisions and external variables.

Such stochastic, sequential decision-making tasks can be formally represented using Markov Decision Processes (MDPs) (e.g.,~\cite{baier2008principles}). 
The application is then modeled as a system where transitions between process steps depend on the current state and specific actions, capturing the probabilistic nature of the problem. While MDPs effectively capture possible decisions, they lack native explainability mechanisms to communicate the significance of specific states (or sequences thereof) to a human observer.

\emph{Attribution-based explanations} offer a potential solution to this problem by assigning importance scores to process elements based on their contribution to the final outcome. For a bank consultant or an applicant, such methods could provide intuitive answers to questions such as: ``\emph{Which specific steps in this application process were most critical in securing the loan approval?}''. Unfortunately, existing attribution techniques are primarily designed for static, one-step prediction tasks~\cite{bodria2023benchmarking,molnar2020interpretable} and fail to transfer to complex sequential decision-making settings.

\paragraph{Contributions.} To bridge this gap, we introduce a novel framework for attribution-based explanations specifically designed for decision-making in MDPs. Unlike existing methods, our approach accounts for the long-term impact of transitions by evaluating the importance of a state relative to a specific strategy (or the set of all possible strategies) for reaching a goal state. Specifically, we define an attribution measure that quantifies state importance based on the reachability probabilities of optimal strategies, ensuring that explanations reflect the most efficient paths to reach a goal. We further extend this notion from individual states to execution paths and propose two optimization approaches to efficiently compute importance scores. Finally, we evaluate our approach on five case studies, demonstrating its feasibility on complex sequential decision-making tasks modeled by MDPs, with thousands of states and tens of thousands of transitions.

\paragraph{Outline.}
\Cref{sec:related_work} reviews related work on attribution-based explanations and \cref{sec:prelim} presents background on MDPs, optimization and attribution-based explanations in Explainable AI.
Our attribution-based explanations for MDPs are formalized in \cref{sec:problem_statement}, and our encoding for computing them is presented in \cref{sec:algorithm}. We evaluate the performance and stability of our encoding in \cref{sec:experiments} and conclude in \cref{sec:conclusion}.


\section{Related Work}
\label{sec:related_work}
In the context of attribution-based explanations, Shapley values~\cite{shapley1953value} have emerged as a dominant theoretical framework. Originally developed to determine fair contributions in cooperative games, they have been widely adapted to machine learning to quantify the impact of input features on model outputs~\cite{LundbergL17}.
Closer to our work, \cite{baier2021responsibility} leveraged Shapley values to propose a notion of attribution that captures entire paths in parametric Markov Chains and, later, extended this notion to general transition systems by transforming them into ``one-shot'' cooperative games~\cite{baier2025responsibility}. Finally, \cite{triantafyllou2021blame} proposed several attribution techniques, including the Shapley value and Banzhaf index~\cite{Banzhaf_1965_5380}, to assign blame to actors in a sequential cooperative multi-agent setting.
These works, however, either do not deal with MDPs, or abstract from the MDP structure and do not leverage reasoning about strategies to determine importance scores. 

This paper approaches the problem from a different angle and exploits the structure of MDPs to provide explanations for sequential decision making in MDPs. To this aim, we
devise novel methods to compute attribution scores based on strategy synthesis. While strategies have been used to provide explanations for MDPs~\cite{KobialkaGLATJ25}, their work targeted counterfactual explanations --- not attribution-based explanations --- and is as such complementary to this work.


\section{Preliminaries}
\label{sec:prelim}

For a finite set $X$, let $\Distr(X)$ be the set of all 
\emph{distri\-butions over $X$}, which are functions $\mu: X \to [0,1]$ with $\sum_{x  {\in} X} \mu(x) {=} 1$; we set $\textit{supp}(\mu)=\{x{\in} X\,|\,\mu(x){>}0\}$.

\subsection{Markov Decision Processes}\label{sec:mdp}
A \emph{Markov Decision Process} (MDP) $\mdp$ is a tuple $\mdptup$ with a finite non-empty set $\states$ of states, a finite set $\actions$ of actions, an initial state $\sinit \in \states$, and a partial transition function $\transitions: \states \times \actions \rightharpoonup \Distr(\states)$.
An action $a \in \actions$ is \emph{enabled} in state $s \in \states$, if $\transitions(s, a)$ is defined.
Let $\actions(s)$ be the set of enabled actions in state $s$; we require that $\actions(s)\not=\emptyset$ for all $s \in \states$.
Transition probabilities $\transitions(s,a)(s')$ are abbreviated  $\transitions(s,a,s')$.

Let $\mdp=\mdptup$ be an MDP. An \emph{infinite path} is an in\-finite sequence $\tau = s_0, a_0,\allowbreak s_1, {\dots}$ of alternating states $s_i\in\states$ and actions $a_i \in A(s_i)$ with $\transitions(s_i,a_i, s_{i+1})>0$ for all $i \geq 0$. A \emph{finite path} is a non-empty finite prefix $\tau=s_0,\ldots,s_n$ of an infinite path; $\tau$ is \emph{simple} if $s_i \neq s_j$ for all $0 \leq i < j \leq n$, and 
its \emph{cylinder set} $\textit{Cyl}(\tau)$
consists of all infinite paths extending $\tau$.  Let $\paths_{\mdp}(s)$ and $\paths^{\textit{fin}}_{\mdp}(s)$ be the sets of all
infinite resp. finite paths of $\mdp$ starting in $s \in \states$. 
A state $t \in \states$ is \emph{reachable} from
$s \in \states$, if a finite path starts in
$s$ and ends in $t$; let $\Reach(t)$ be the set of all states from
which $t$ is reachable.

\emph{Strategies} (or \emph{schedulers}) resolve
nondeterminism in MDPs by deciding which actions to take.
We consider \emph{stochastic memoryless} strategies $\strategy: \states \to \Distr(\actions)$ with $\textit{supp}(\strategy(s))\subseteq\allowbreak\actions(s)$ for all $s\in\states$,
forming the set $\Sigma_\mdp$.
We abbreviate strategy probabilities
$\strategy(s)(a)$ as $\strategy(s,a)$, and call $\strategy$ 
\emph{deterministic} if $|\textit{supp}(\strategy(s))|=1$
for any $s \in \states$.    
\emph{Finite memory} strategies extend memoryless ones with a finite set $M$ of memory modes, differentiating how the current state is reached. Finite memory strategies for an MDP with states $\states$ can be simulated by memoryless strategies on an MDP with state space $\states \times M$. Therefore, we use the memoryless notation also for finite memory strategies.

For a strategy $\sigma$ of $\mdp=\mdptup$, its \emph{induced Markov Chain} is
$\mdp^\sigma =\langle \states, \sinit, \transitions^{\sigma} \rangle$, where $\transitions^{\sigma} \colon \states \times \states \to [0,1]$ with $\transitions^\sigma(s, s') = \sum_{a \in \actions} \sigma(s)(a)\cdot\transitions(s, a, s')$ for all $s,s' \in \states$.
With $\dtmc = \mdp^\sigma$ we associate the
\emph{probability space} $(\paths_\dtmc(\sinit), \allowbreak \{\bigcup_{\tau \in R} \textit{Cyl}(\tau) \mid R~\subseteq~\paths^{\textit{fin}}_{\dtmc}(\sinit)\}, \allowbreak \Pr_\dtmc(\sinit))$, where the probability of the cylinder set of a finite path $\tau = s_0\ldots s_n$ is $\Pr_\dtmc(\sinit)(\textit{Cyl}(\tau)) = \prod_{i=0}^{n-1} \transitions^{\sigma}(s_{i}, s_{i+1})$.
For a state $t$, let $\Sigma^t_\mdp$ be the set of all strategies $\strategy\in\Sigma_{\mdp}$ with a positive probability of reaching $t$ in $\mdp^{\sigma}$.



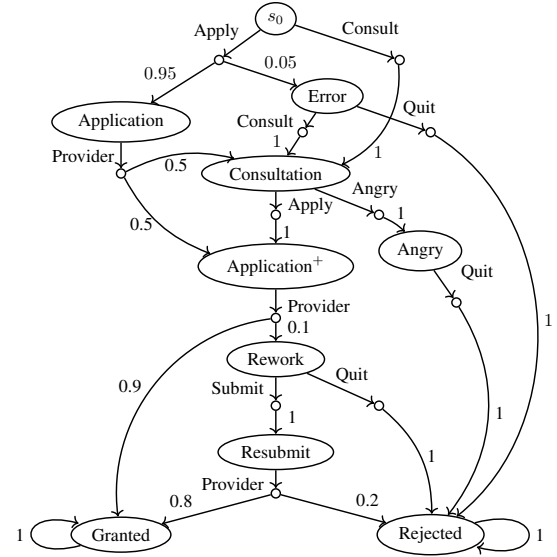
\begin{figure}
    \centering
        \resizebox{0.9\linewidth}{!}{
        \tikzstyle{line} =[line width=2\unitlength, black!60, >=latex]
\begin{tikzpicture}[
  rflow/.style = {thick,->},
  circstyle/.style = {shape=circle,draw=black, minimum size=5pt, inner sep=0pt,thick},
  eclipstyle/.style = {shape=ellipse,draw=black,thick},
  ]
    \node[eclipstyle] (s0) at (0,0) {$s_0$};
    \node[circstyle] (s0consult) at (2.4,-0.8) {};
    \node[circstyle] (s0apply) at (-1.1,-0.8) {};
    
    \node[eclipstyle] (Application) at (-3,-2) {Application};
    \node[circstyle] (applicationProvider) at (-3,-3) {};

    \node[eclipstyle] (Error) at (1,-1.5) {Error};
    \node[circstyle] (errorConsult) at (0.52,-2.2) {};
    \node[circstyle] (errorQuit) at (3.0,-2.2) {};

    \node[eclipstyle] (Consultation) at (0,-3) {Consultation};
    \node[circstyle] (consultationApply) at (0,-3.8) {};
    \node[circstyle] (consultationQuit) at (2,-3.8) {};

    \node[eclipstyle] (Angry) at (2.8,-4.5) {Angry};
    \node[circstyle] (angryQuit) at (3.5,-5.5) {};

    \node[eclipstyle] (Application+) at (0,-4.8) {$\text{Application}^+$};
    \node[circstyle] (application+Provider) at (0,-5.8) {};

    \node[eclipstyle] (Rework) at (0,-6.6) {Rework};
    \node[circstyle] (reworkSubmit) at (0,-7.5) {};
    \node[circstyle] (reworkQuit) at (2,-7.5) {};

    \node[eclipstyle] (Resubmit) at (0,-8.4) {Resubmit};
    \node[circstyle] (resubmitProvider) at (0,-9.2) {};

    \node[eclipstyle] (Granted) at (-3,-10) {Granted} ;
    \node[eclipstyle] (Rejected) at (3,-10) {Rejected} ;

    \path [rflow] (s0) edge node[left, near start, xshift=-1mm,yshift=1mm] {Apply} (s0apply);
    \path [rflow] (s0) edge node[right, near start, xshift=3mm, yshift = 1mm] {Consult} (s0consult);
    \path [rflow] (s0apply) edge node[left, near start, xshift=-3mm] {$0.95$} (Application);
    \path [rflow] (s0apply) edge node[right, near start, xshift=3mm, yshift=1mm] {$0.05$} (Error);
    \path [rflow] (s0consult) edge[bend left = 35] node[near end,xshift=2mm,yshift=-1mm] {$1$} (Consultation);

    \path [rflow] (Application) edge node[left] {Provider} (applicationProvider);
    \path [rflow] (applicationProvider) edge[bend right = 20] node[left] {0.5} (Application+);
    \path [rflow] (applicationProvider) edge[bend left = 20, near start] node[right, xshift=1mm, yshift=-1mm] {0.5} (Consultation);

    \path [rflow] (Error) edge node[left, xshift=-2mm] {Consult} (errorConsult);
    \path [rflow] (Error) edge node[right, xshift=1mm, yshift=2mm] {Quit} (errorQuit);
    \path [rflow] (errorConsult) edge node[left, xshift=-1mm,yshift=1mm] {$1$} (Consultation);
    \path [rflow] (errorQuit) edge[bend left = 55] node[right] {$1$} (Rejected);

    \path [rflow] (Application+) edge node[right, xshift=1mm,yshift=-1mm] {Provider} (application+Provider);
    \path [rflow] (application+Provider) edge node[right, xshift=1mm,yshift=1mm] {0.1} (Rework);
    \path [rflow] (application+Provider) edge[bend right = 45] node[left,xshift=-1mm] {0.9} (Granted);

    \path [rflow] (Rework) edge node[left, near start,xshift=-1mm,yshift=-1mm] {Submit} (reworkSubmit);
    \path [rflow] (Rework) edge node[right, near start, xshift=1mm, yshift=1mm] {Quit} (reworkQuit);
    \path [rflow] (reworkSubmit) edge node[right, xshift=1mm,yshift=1mm] {1} (Resubmit);
    \path [rflow] (reworkQuit) edge[bend left = 25] node[right] {1} (Rejected);

    \path [rflow] (Consultation) edge node[right, near start, xshift=3mm,yshift=1mm] {Angry} (consultationQuit);
    \path [rflow] (Consultation) edge node[right, xshift=1mm,yshift=-1mm] {Apply} (consultationApply);
    \path [rflow] (consultationQuit) edge[bend left = 10] node[right, yshift=2mm, xshift=-1mm] {1} (Angry);
    \path [rflow] (consultationApply) edge node[right] {1} (Application+);

    \path [rflow] (Angry) edge node[right, near start, xshift=3mm,yshift=1mm] {Quit} (angryQuit);
    \path [rflow] (angryQuit) edge[bend left = 35] node[right] {1} (Rejected);

    \path [rflow] (Resubmit) edge node[left, xshift=-1mm,yshift=-1mm] {Provider} (resubmitProvider);
    \path [rflow] (resubmitProvider) edge node[right, near end, xshift=-2mm,yshift=3mm] {0.2} (Rejected);
    \path [rflow] (resubmitProvider) edge node[left, near end, xshift=2mm,yshift=3mm] {0.8} (Granted);

    \path [rflow] (Granted) edge[loop left] node[left] {1} (Granted);
    \path [rflow] (Rejected) edge[loop right] node[right] {1} (Rejected);
\end{tikzpicture}
    }
    \caption{An MDP model of a loan application process. The bank can only perform the Provider action, which is purely stochastic; thus, the applicant has full control over non-deterministic choices.}
    \label{fig:mdp_ex}
\end{figure}

\begin{example}
    \label{ex:loan}
    The MDP in \Cref{fig:mdp_ex} models a loan application process. Starting in $s_0$,
    the applicant decides between taking a consultation or directly submitting an application.
    If applying directly, the application is error-free with a chance of $95\%$, and the bank decides whether a consultation is needed or the application can be considered directly.
    If an error occurred when applying, the Error state is reached and a consultation is needed.
    In the consultation, the user either applies again, or they become angry and quit the application process.
    In the former case, the submitted application is evaluated by the bank, and may be reworked  one more time before a final decision is made.
\end{example}




\subsection{Non-linear Optimization}
\emph{Mixed Integer Quadratically Constrained Quadratic Problems} (MIQCQPs)~\cite{billionnet2016exact} are non-linear optimization problems, where both objective function and constraints are quadratic functions over real and integer variables.
Formally, MIQCQPs are defined as:

\begin{equation*}
    \begin{array}{lll}
        \min \quad & f_0(x) \\
        \text{subject to } & f_i(x) \leq b_i & \textit{for\ \ } i = 1, \dots, m,\\
        & 0 \leq x_j \leq u_j &  \textit{for\ \ } x_j \in V_{\Z} \cup V_{\R}, \\
        & x_j \in \Z &  \textit{for\ \ } x_j \in V_{\Z}, \\
        & x_j \in \R &  \textit{for\ \ } x_j \in V_{\R},
    \end{array}
\end{equation*}
using integer-valued ($V_\Z$) and real-valued ($V_\R$) variables $x = (x_1, \dots, x_n)$ ($n\in\N$) in the functions $f_i(x)=x^T Q_i x+c^T_i x$ with symmetric matrices  $Q_i\in\R^{n \times n}$ and constants $c_i\in\R^n$ for all $i\in\{0, \dots, m\}$ ($m\in\N$).
Upper bounds align with the domain of the constrained variable, i.e., $u_j\in\Z$ for $x_j\in V_{\Z}$, and $u_j\in\R$ for $x_j\in V_{\R}$.
MIQCQP are generally non-convex and thus hard to solve~\cite{billionnet2016exact,garey1979computers}.

We consider
\emph{hierarchical} objectives~\cite{anandalingam1992hierarchical}, such that the objective function $f_0(x)$ maps to a tuple and higher-indexed objectives are optimized under optimally-solved lower-indexed objectives, i.e., $f_0(x)_1$ is optimized with the additional constraint that $f_0(x)_0$ is optimal.





\subsection{Attribution-based Explanations}

An attribution-based explanation explains the outcome of an AI model
by assigning a numerical \emph{score} to its inputs. Intuitively,
these \emph{attribution scores} represent the degree to which each
input influences the observed result. In Machine Learning, attribution
methods such as LIME, SHAP and
gradient-scoring~\cite{Ribeiro0G16,LundbergL17,SelvarajuCDVPB17,sundararajan2020many}
assign a score to
each input feature (e.g., pixels, words) to reflect their contribution
towards an output (e.g., a classification); see also the survey~\cite{molnar2020interpretable}.
Although well-suited for ML, these methods are generally not applicable for MDPs because they attribute importance to static input features at a single point in time, failing to account for crucial factors in sequential decision-making, such as sequential dependencies, long-term rewards, and stochastic state transitions.


\section{Attribution-based Explanations for MDPs}
\label{sec:problem_statement}



As a first contribution, we give a formal characterization of attributions for MDPs, establishing a quantitative basis for measuring the importance of states and paths, such that questions about their influence in reaching a target can be answered.

\textbf{LTL notation. }In the sequel, we will use notation from Linear Temporal Logic (LTL)~\cite{pnueli1977temporal}, which we briefly recall. Based on a set of atomic propositions, e.g., a set $\states$ of state names, LTL formulae $\varphi$ are inductively defined by the abstract grammar $\varphi \Coloneqq \top \mid s \mid \varphi \land \varphi \mid \lnot \varphi \mid \bigcirc \varphi \mid \varphi \U \varphi$ with $s \in \states$.
The formula $\bigcirc \varphi$ expresses that $\varphi$ holds in the \emph{next} state, $\varphi_1\U \varphi_2$ that $\varphi_1$ holds \emph{until} $\varphi_2$ gets true, and
$\lozenge \varphi$ that $\varphi$ will \emph{eventually} hold (i.e., $\top \U \varphi$).
For an MDP $\mdp$, a strategy $\strategy \in \strategies_{\mdp}$, and an LTL formula $\varphi$, $Pr_{\mdp^\strategy}(\varphi) = Pr_{\mdp^\strategy}\{\tau \in \paths_{\mdp}(s_0) \mid \tau \models \varphi\}$~\cite{baier2008principles}. 

\subsection{Our formal attribution framework}

As our first contribution, in order to provide a comprehensive view of importance, we formalize four attribution problems. We first address the importance of individual states (\textbf{P1}) and paths (\textbf{P3}) under a fixed strategy, capturing the realization of a specific behavior. We then generalize these notions to quantify the importance of states (\textbf{P2}) and paths (\textbf{P4}) across all strategies, revealing the global importance of these elements within the MDP. 
We begin with quantifying the importance of an individual state within the context of a fixed strategy.

\begin{problemstatement}{P1}
Consider an MDP $\mdp = \mdptup$, a target state $t\in\states$, and a strategy $\sigma \in \Sigma^t_\mdp$.
\emph{How important is a given state $s \in \states$ for reaching $t$ in $\mdp^\strategy$?
 }
\end{problemstatement}

\noindent Intuitively, a state $s\in\states$ is \emph{important} under strategy $\sigma$ for reaching state $t \in \states$, if the induced Markov Chain $\mdp^\strategy$ has a high probability of visiting $s$ before reaching $t$. We denote this probability by
$\Pr_{\mdp^\sigma}(\cond{s}) \coloneqq \Pr_{\mdp^\sigma}(\lozenge t \land (\lnot t \U s))$, 
to capture the probability of eventually reaching $t$ and visiting $s$ before reaching $t$.
If $\mdp$ is loop-free and $s$ is never reached after $t$ on any simple path, then $P_{\mdp^\sigma}(\cond{s})$ reduces to $P_{\mdp^\sigma}(\lozenge t \land \lozenge s)$.


Given the above, we define the importance of a state $s$ for reaching $t$ in MDP $\mdp$ under strategy $\strategy$ as:
\begin{equation}
    \imp_{\mdp^{\sigma}}^{\sinit, t}(s) \coloneqq \frac{\Pr_{\mdp^\sigma}(\cond{s})}{\Pr_{\mdp^\sigma}(\lozenge t)}. \label{eq:importance}
\end{equation}

\noindent To ensure that our importance measure is well-defined, we require that $\strategy\in \strategies^t_\mdp$ (i.e., the probability of reaching $t$ under $\strategy$ is positive).
Note that $\imp_{\mdp^{\sigma}}^{\sinit, t}(s) \in [0,1]$ for any $s \in \states$ and $\strategy \in \strategies^t_\mdp$; $\imp_{\mdp^{\sigma}}^{\sinit, t}(s) = 0$ indicates that $s$ is never visited before reaching $t$, and $\imp_{\mdp^{\sigma}}^{\sinit, t}(s) = 1$ indicates that $s$ is always visited before reaching $t$.
Thus the initial state $s_0$ and the target state $t$ receive an importance of $1$ under all strategies, as they are guaranteed to be visited. However, the importance of the other states might depend on the strategy at hand.

\begin{problemstatement}{P2}
Consider an MDP $\mdp  =  \mdptup$ and a target state $t$.
\emph{How important is a given state $s\in \states$ for reaching $t$?}
\end{problemstatement}

\noindent In contrast to \textbf{P1}, no strategy is assumed  in \textbf{P2}. 
Here, the importance of a state $s \in \states$ depends on whether $s$ is important for reaching $t$ globally.
Intuitively, if $s$ is never visited before reaching $t$ under any strategy, then $s$ is not important. In fact, visiting $s$ might even be detrimental for reaching $t$, and thus receive an importance score of $0$.
As strategies can maximize and minimize the probability of visiting $s$ before reaching $t$, we express importance for \textbf{P2} as a tuple:
\begin{equation*}
    \imp_{\mdp}^{\sinit, t}(s) \coloneqq \big(\min_{\sigma \in \Sigma_\mdp^t} \imp_{\mdp^{\sigma}}^{\sinit, t}(s), \max_{\sigma \in \Sigma_\mdp^t} \imp_{\mdp^{\sigma}}^{\sinit, t}(s) \big).
\end{equation*}
Obviously, importance scores between the lower and upper bound can be realized by mixing the minimizing and maximizing strategies.
However, strategies optimizing $\imp_{\mdp^{\sigma}}^{\sinit, t}(s)$ might require up to one bit of memory to determine whether $s$ was visited before reaching $t$.
If lower and upper bounds align, the probability to visit $s$ before $t$ equals for all $\sigma\in\strategies^t_\mdp$.




In P1, the notion of importance normalizes the probability of visiting  $s$ before
$t$ by dividing by the total probability of reaching $t$.
In addition, we also define \emph{absolute importance} without normalization as the
tuple 
$\big(\min_{\sigma \in \Sigma_\mdp^t} \Pr_{\mdp^\sigma}(\cond{s}),
\max_{\sigma \in \Sigma_\mdp^t} \Pr_{\mdp^\sigma}(\cond{s}) \big)$ of minimal and maximal probabilities
of reaching $t$ via $s$.


Available model checking tools apply \emph{value iteration} to monotonically decrease resp.\ increase the probability of reaching $t$ by local strategy modifications, adapting decisions at individual states. However, the importance of a state is defined as a fraction of \emph{two probabilities} defined by the \emph{same strategy}, and as such it cannot be directly computed by available tools.
  Intuitively, value iteration can be applied to both the nominator and the denominator in isolation, but monotonicity cannot be assured for their fraction.
  See Appendix~\ref{app:monotonicity} for an illustrating example.




We now extend the importance of states for reaching a state $t$, to the importance of paths, where several states need to be visited before reaching $t$.

\begin{problemstatement}{P3}
Consider an MDP $\mdp = \mdptup$, a target state $t$, 
a simple path $\tau = \langle\sinit, a_0, \dots, s_n\rangle$ with $s_i \neq t$ for $0 \leq i \leq n$ and $s_n \in \Reach(t)$, and a strategy $\strategy$.
\emph{How important is $\tau$ for reaching $t$ under $\sigma$?}
\end{problemstatement}

\noindent Intuitively, a path $\tau$ is important for reaching $t$ under strategy $\strategy$, if the induced Markov Chain $\mdp^\strategy$ is likely to follow $\tau$ before reaching $t$.

For a path $\tau$ of the above form, we express the probability of  reaching $t$ after following $\tau$ as $\Pr_{\mdp^\sigma}(\condPath{\tau}) \coloneqq \Pr_{\mdp^\sigma}(\lozenge t \land (\lnot t \U \bigcirc(s_1 \land (\dots \land (\bigcirc s_n) \dots ))))$, which simplifies to $\Pr_{\mdp^\sigma}(\lozenge t \land \bigcirc(s_1 \land (\dots \land (\bigcirc s_n) \dots )))$.

The  definition of the importance of a state in \Cref{eq:importance} is extended to paths as follows.
For a strategy $\strategy$ of the MDP $\mdp$, the importance of $\tau$ for reaching $t$ is defined as: 
\begin{equation*}
    \imp_{\mdp^{\sigma}}^{\sinit, t}(\tau) \coloneqq \frac{\Pr_{\mdp^\sigma}(\condPath{\tau})}{\Pr_{\mdp^\sigma}(\lozenge t)}.
\end{equation*}
Though we could allow paths visiting $t$ or with $s_n \not\in \Reach(t)$,  the importance score would then default to $0$.
Therefore, we constrain $\tau$ not to visit $t$ and end in some $s_n \in \Reach(t)$.

In the absence of a fixed strategy, the importance of path $\tau$ is lower and upper bounded, expressing the global importance of $\tau$ for reaching $t$.

\begin{problemstatement}{P4}
Consider an MDP $\mdp = \mdptup$, a target state $t$, and
a simple path $\tau = \langle\sinit, a_0, \dots, s_n\rangle$, with $s_i \neq t$ for $0 \leq i \leq n$ and $s_n \in \Reach(t)$.
\emph{How important is $\tau$ for reaching $t$?}
\end{problemstatement}

\noindent
As \textbf{P4} abstracts from the given strategy in \textbf{P3}, we define the importance of a path as a tuple of lower and upper bounds on the importance of $\tau$ for reaching $t$ under any strategy:
\begin{equation*}
    \imp_{\mdp}^{\sinit, t}(\tau) \coloneqq \big( \min_{\strategy \in \strategies_\mdp^\tau} \imp_{\mdp^{\strategy}}^{\sinit, t}(\tau), \max_{\strategy \in \strategies_\mdp^\tau} \imp_{\mdp^{\strategy}}^{\sinit, t}(\tau) \big)\, ,
\end{equation*}
where the set $\strategies_\mdp^\tau$ of strategies from $\strategies_\mdp^t$ follows the action choices in $\tau$, i.e., $\strategies_\mdp^\tau = \{ \strategy \in \strategies_\mdp^t \mid \strategy(s_i) = a_i \text{ for } 0 \leq i < n \} $.
We restrict the strategies over which the importance of a path is optimized to $\strategies_\mdp^\tau$, as the lower bound is trivially realizable by not following $\tau$.

We conclude this section with \Cref{ex:loan_importance} to illustrate our definitions on the loan application from \Cref{ex:loan}.

\begin{example}
    \label{ex:loan_importance}
    A bank employee restructuring the loan application process from
    \Cref{ex:loan} wants to determine the importance of the
    consultation state for a successful application.  Our attribution framework would show that the importance
    of
    $\text{Application}^+$ is lower and upper bounded by one, as every
    successful application must visit that state,
    $\imp_{\mdp}^{\sinit, \text{Granted}}(\text{Application}^+) =
    (1,1)$. The importance of
    $\text{Angry}$, which is detrimental for a successful application,
    is lower and upper bounded by $0$, $\imp_{\mdp}^{\sinit,
      \text{Granted}}(\text{Angry}) =
    (0,0)$.  The importance of the path $\tau = \langle s_0, \text{Apply},
    \text{Application}\rangle$ is lower bounded by
    $0.9$, and upper bounded by $1$: $\imp_{\mdp}^{\sinit,
      \text{Granted}}(\tau) = (0.9, 1)$. 
        
\end{example}

Such importance bounds provide a formal yet interpretable measure of how specific states and paths contribute to -- or hinder -- the reachability of a target. In the following section, we detail our novel approach used to compute these intervals.

\section{Computing Importance}
\label{sec:algorithm}
After defining importance, in this section we propose a framework to compute its value. While our method is applicable to both states and paths, the following derivations focus on states, and we discuss the generalization to paths at the end of this section.
We introduce three encodings: (i) \gqp, a non-linear encoding optimizing importance over all strategies, (ii) \qp, a non-linear encoding optimizing over reachability-optimal strategies, and (iii) \lp, a linear encoding optimizing importance over reachability-optimal strategies.
Theoretical results connect the encodings: \cref{lem:nonconvex} establishes that \gqp is computationally difficult and Theorem~\ref{thrm:sound_complete} proves its correctness; \cref{lem:pure_strat} enables the cheaper linear encoding \lp.

Assume an MDP $\mdp^* = \langle \states^*, \actions, \sinit, \transitions^* \rangle$ and let $\hat{s} \in S^*\setminus\{s_0\}$ be the state of interest for reaching state $t\in\states^*$.
Strategies optimizing $\imp_{\mdp^*}^{s_0, t}(\hat{s})$ require 1 bit of memory to remember whether $\hat{s}$ was visited. 
An initial preprocessing step introduces this memory,
analogously to the memory construction in LTL model checking~\cite{baier2008principles}.


\paragraph{Preprocessing.} 
We construct the MDP $\mdp = \langle \states{,} A{,} s_0^\bot{,} \transitions \rangle$ with two copies $s^\top$ and $s^\bot$ for each state $s \in \states^*$: $\states = \{s^\top, s^\bot \mid s \in \states^* \}$.
Index $\top$ indicates that $\hat{s}$ was visited and $\bot$ that $\hat{s}$ was not visited.
Accordingly,
$\transitions(s_1^{\bowtie_1},a,s_2^{\bowtie_2})$ is $\transitions^*(s_1,a,s_2)$ if  $s_2^{\bowtie_2}=\hat{s}^\top$ or ($\bowtie_1=\bowtie_2$ and $s_2\not=\hat{s}$), and $0$ otherwise,
for all $s_1,s_2 \in \states^*$, $a \in \actions$, and $\bowtie_1,\bowtie_2 \in \{\top, \bot\}$.





\paragraph{Encoding as Quadratic Program (\gqp).}
Based on the preprocessed MDP $\mdp = \mdptup$,
\cref{fig:encoding_mixed_scheduler} shows the encoding of the lower bound $\min_{\sigma\in\Sigma_{\mdp^*}^t} \imp_{{\mdp^*}^\sigma}^{s_0, t}(\hat{s})$ on the importance of $\hat{s}$ as the solution to a nonlinear optimization problem.
\gqp uses three types of variables: (1) $p_{sa} \in [0,1]$ stores the probability of choosing action $a \in \actions$ in state $s \in \states$ by the computed strategy, (2) $p_{s, t^{\bowtie}} \in [0,1]$ stores the probability of reaching state $t^{\top}$ resp. $t^{\bot}$ under the computed strategy, and (3) $\tau_s \in \R$ allows to encode weight-decreasing paths to terminal states to enforce smallest fixedpoints.
Note that $p_{s_0^\bot, t^\bot} + p_{s_0^\bot, t^\top}$ is the probability of reaching $t$ from $s_0$ in the non-preprocessed MDP, as all paths either visit or bypass $\hat{s}$.

\begin{figure}
  \small
  \scalebox{0.99}{
    \begin{minipage}{1.0\linewidth}
\begin{align}
  && \min 
  \frac{p_{s_0^\bot, t^\top}}{p_{s_0^\bot, t^\bot} + p_{s_0^\bot, t^\top}} \label{eq:qp_target}  \\
  && \quad \text{subject to}: \nonumber \\
  \forall s \in \states_1:=S\setminus\{t^\bot,t^\top\}:\hspace*{-5ex} && \sum_{a\in\actions(s)} p_{sa}  = 1 \label{eq:qp_sum}\\
  && p_{t^\bot, t^\bot}  = p_{t^\top, t^\top} = 1 \label{eq:qp_ptBB} \\
  && p_{t^\bot, t^\top}  = p_{t^\top, t^\bot} = 0 \label{eq:qp_ptBT} \\
  && p_{s_0^\bot, t^\bot} + p_{s_0^\bot, t^\top} \geq \epsilon \label{eq:qp_lower_bound}\\
  \!\!\forall s \!\in\! S_1,\bowtie \in\!\{\top, \bot\}: &&\hspace*{-2ex} \label{eq:qp_bellman} 
   p_{s, t^{\bowtie}} \! = \! \!\!\sum_{a \in A(\!s\!)} \sum_{s' \in S} p_{sa} {\cdot} \transitions(s,a,s'){\cdot} p_{s', t^{\bowtie}} \\
  \forall s \in \states \setminus T: 
  &&\hspace*{-4ex} \tau_s + 1 \leq \hspace*{-1.5ex}\sum_{a \in \actions(s)} \sum_{s' \in \states} p_{sa} \cdot \transitions(s,a,s') \cdot \tau_{s'} \label{eq:qp_path}
\end{align}
\end{minipage}
}
\vspace{-12pt}
\caption{Encoding \gqp for optimizing  the importance of $\hat{s} \in \states$ in the preprocessed MDP $\mdp$ as a nonlinear optimization problem. \label{fig:encoding_mixed_scheduler}}
\end{figure}

\cref{eq:qp_target} minimizes an objective function that encodes the importance of $\hat{s}$; maximizing corresponds to minimizing the objective multiplied by $-1$.
Constraint~\ref{eq:qp_sum} ensures that the computed strategy is well-defined (i.e., the probabilities sum up to 1).
Constraints~\ref{eq:qp_ptBB}--\ref{eq:qp_ptBT} encode default values for reaching state $t$,
and Constraint~\ref{eq:qp_lower_bound} ensures a positive probability
of reaching $t$ for an arbitrarily small $\epsilon$.
Constraint~\ref{eq:qp_bellman} encodes the Bellman optimality
condition.  Constraint~\ref{eq:qp_path} encodes a state ordering~\cite{wimmer2012minimal} to ensure that 
bottom strongly connected components not containing $t$ reach $t$ with probability $0$, where the set of terminal states $T \subset
\states$ contains all absorbing states where the only available action
is a self-loop.
Constraint~\ref{eq:qp_ptBT}
is not required for correctness
but can reduce the solving time.
 Constraint~\ref{eq:qp_bellman} can be restricted to states reaching $t^\top$ or $t^\bot$, setting $p_{s,t^{\bowtie}}$ to $0$ for states not reaching $t^{\bowtie}$.
All values in $p_{sa}$ can be transformed into a strategy \emph{with memory} on the non-preprocessed MDP, introducing $\bowtie ~\in \{\top, \bot\}$ as a memory mode to remember whether $\hat{s}$ was visited.

Let $P$ denote the optimization problem defined in \cref{fig:encoding_mixed_scheduler}.
The Bellman-optimality encoding in Constraint~\ref{eq:qp_bellman} is non-convex, preventing polynomial-time solutions:
\begin{lemma}
\label{lem:nonconvex}
    P is in general non-convex.
\end{lemma}
\begin{proofsketch}
  An optimization problem is convex, if its objective and all
  constraints are convex~\cite{boyd2004convex}.  As
  Constraint~\ref{eq:qp_bellman} is in general non-convex, see
  e.g.~\cite{KobialkaGLATJ25}, the theorem follows.
\end{proofsketch}

\begin{lemma}
  \label{lem:optimal}
  Assume a solution $\nu$ to $P$, which maps each variable $v$ to value $\nu(v)$. 
  Further, let $\strategy$ be the strategy defined by $\strategy(s,a) = \nu(p_{sa})$, for all $\state \in \states$ and $\action \in \actions$.
  Then, the value of the objective function equals $\min_{\sigma \in \Sigma_{\mdp^*}^t} \imp_{{\mdp^*}^{\sigma}}^{\sinit, t}(\hat{s})$.
\end{lemma}

The following theorem states that the encoding facefully captures the importance of $\hat{s}$.
\begin{theorem}[Soundness and Completeness]
    \label{thrm:sound_complete}
    $P$ is feasible iff the importance of state $\hat{s}$ is well defined.
\end{theorem}

The proofs for \Cref{lem:optimal} and \Cref{thrm:sound_complete} are contained in Appendix~\ref{app:encoding_properties}.


\paragraph{Importance for optimal reachability.} 
In practice, we are typically interested in the importance of $\hat{s}$ under strategies that maximize the probability of reaching $t$.
Then we can restrict the importance computation to strategies from
$\strategies^* = \{ \strategy \in \strategies_{\mdp^*} \mid P_{{\mdp^*}^{\strategy}}(\lozenge t) \geq P_{{\mdp^*}^{\strategy'}}(\lozenge t) \text{ for all } \strategy' \in \strategies_{\mdp^*}\}$. 
This yields a constant value for the denominator in \cref{eq:qp_target}, and the problem reduces to minimizing $p_{s_0^\bot,t^\top}$.

When only considering  optimal reachability,
\gqp can be further simplified
 as reachability probabilities from strongly connected components can be  lower-bounded directly.
We extend the problem to multi-objective optimization, minimizing the sum of reachability probabilities (under the maximizing schedulers from $\Sigma^*$) before optimizing for importance:
\[
\min \big(\big(\sum_{s \in \states,  \bowtie \in \{\top, \bot\}} p_{s,t^{\bowtie}}\big), \quad 
p_{s_0^\bot, t^\top}\big).
\]
We further simplify computations  by fixing the optimal reachability probability $p^*$ and lower-bounding the reachability of $t^\top$ and $t^\bot$ over all possible actions to replace Constraint~\ref{eq:qp_path}:
\begin{eqnarray}
  p_{s_0^\bot, t^\bot} + p_{s_0^\bot, t^\top} =&&\hspace*{-4ex} p^* \label{eq:qp2_lower_bound} \nonumber \\
  \forall s \in \states_1, a \in \actions: 
  \hspace*{-1ex}\sum_{\bowtie \in \{\top, \bot\}} p_{s, t^{\bowtie}} \geq &&\hspace*{-6ex} \sum_{\substack{s' \in \states, \\ \bowtie \in \{\top, \bot\}}} \hspace*{-2ex}\transitions(s, a, s') \cdot p_{s', t^{\bowtie}} \nonumber
\end{eqnarray}
This encoding, denoted \qp, does not guarantee that the computed strategy is
loop-free, but ensures that the assigned probability corresponds to
actually realizable probabilities. The strategy can always be
repaired to satisfy a state ordering.

We show that restricting strategies to reachability-optimal strategies simplifies the problem of computing importance:

\begin{lemma}
    \label{lem:pure_strat}
    For an MDP $\mdp$, target state $t$, state $\hat{s} \in \states$ and simple path $\tau$, there exists a deterministic strategy optimizing $\imp_{\mdp^{\sigma}}^{\sinit, t}(\hat{s})$ and $\imp_{\mdp^{\sigma}}^{\sinit, t}(\tau)$ over all strategies in $\strategies^*$.
\end{lemma}
See Appendix~\ref{app:deterministic_strategy} for a proof of \Cref{lem:pure_strat}.
It follows that \qp can be adapted to a mixed-integer linear program, denoted as \lp, without quadratic components,
in which the variables $p_{sa}$ are turned into binary variables (i.e., from $\{0,1\}$).
The optimization objective only optimizes for the importance: $\min 
p_{s_0^\bot, t^\top}$.
Constraints~\ref{eq:qp_bellman} and~\ref{eq:qp_path} are replaced by:
\begin{align}
  & p_{s_0^\bot, t^\bot} + p_{s_0^\bot, t^\top} = p^* \label{eq:lp_lower_bound} \\
  \forall s \in S_1, a \in \actions(s), \bowtie \in \{\top, \bot\}: & \label{eq:lp_bellman}\\ 
   p_{s, t^{\bowtie}}  \leq   \sum_{s' \in S} \transitions(s,a,&s')\cdot p_{s', t^{\bowtie}} + (1 - p_{sa}) \nonumber\\
  \forall s \in S_1 \setminus T, a \in \actions(s): & \label{eq:lp_path}\\
  \tau_s + 1 \leq \sum_{s' \in \states} \transitions(s,a,&s') \cdot \tau_{s'} + (M - M \cdot p_{sa}) \nonumber
\end{align}
Constraint~\ref{eq:lp_lower_bound} ensures that optimal reachability is realized.
Constraint~\ref{eq:lp_bellman} adapts the Bellman optimality equations to discrete actions; only for the chosen action in each state is the constraint not trivially satisfied.
To realize the optimal reachability, those constraints must be satisfied with equation.
Constraint~\ref{eq:lp_path} ensures that a weight-decreasing path exists under the computed strategy, where $M$ is a large constant.


\paragraph{\textit{Remark.}} All the presented encodings can be extended to compute the importance of a simple path $\tau = \langle\sinit, a_0, \dots, s_n\rangle$, with $s_i \neq t$.
The importance of $\tau$ can be expressed as the importance of $s_n$ weighted with the probability of observing $\tau$.
By partially assigning $p_{s_ia_i}=1$ to maximize the probability of following $\tau$ and then solving for the importance of $s_n$, the importance of $\tau$ is the result multiplied with $\prod \transitions^*(s_i, a_i, s_{i+1})$.
The optimization problem is only well defined, and thus feasbile, 
for strategies reaching $t$ with positive probability, i.e., it is ill defined 
if $s_n \not\in \Reach(t)$.


\section{Experimental Evaluation}
\label{sec:experiments}

We now show that our approach is able to compute importance scores efficiently across five case-studies.
\emph{GrepS} records a programming skill evaluation service where customers (1) sign-up, (2) solve programming tasks, and (3) release their test results~\cite{kobialka2024user}.
\emph{BPIC12}~\cite{bpi2012} and \emph{BPIC17}~\cite{bpi2017} are datasets for loan applications
released by the IEEE Task Force on Process Mining.\footnote{\url{https://www.tf-pm.org/competitions-awards/bpi-challenge}}
\emph{MSSD} is the \emph{Music Streaming Sessions Dataset}~\cite{brost2019music} released by Spotify (\emph{Mini} version with $10\,000$ listening sessions).
\emph{Epidemic}~\cite{kazemi2025counterfactual} is a synthetic benchmark on vaccination strategies.


We used passive automata learning~\cite{mao2016learning,DBLP:journals/isse/MuskardinAPPT22} to learn MDPs
from the traces, after applying standard preprocessing techniques.
For MSSD, we construct models that vary in the number of traces; e.g.,
MSSD10 includes $10\%$ of the traces.  For Epidemic, we vary the
maximal population size; e.g., E6 considers a maximum population
size of 6.
The number of states is significantly larger in the MSSD models than
in the other models: MSSD10 contains on average 36 times more states
than BPIC12 or BPIC17, and has a 70 times higher maximum of outgoing transitions. Epidemic
has a similar number of states as MSSD, but the number of transitions is 6 times greater than that of MSSD.  However, the models are less
connected, the largest Epidemic model has a maximum of 48 outgoing transitions, and
MSSD100 of over $5\,600$.
Each model contains a unique target state $t$ that represents success.
For example, in BPIC12 and BPIC17, the loan is granted to the user in  $t$.

For Greps, BPIC12, and BPIC17, we compute the importance of every
state in the model.  For the MSSD models, we
sample 10 states per model.
For the Epidemic models, we use the proposed initial state by~\cite{kazemi2025counterfactual} to compute the importance of the number of available vaccinations for reaching a state without unvaccinated individuals.

All experiments were performed on a workstation with $128\unit{GB}$ memory and a 64-core AMD EPYC processor.
Our implementation and results are
online.\footnote{\url{https://github.com/explainableMDPs/mdp_attributive}}
Python 3.10.12 is used to interface with  Gurobi\footnote{\url{https://www.gurobi.com/}} (v.\,11.0.3), used to solve the optimization problems.
Each run of the solver is restricted to $4\unit{GB}$ of memory and $1$ hour, in all experiments we set $M=10^{16}$ and $\epsilon=10^{-4}$.

\paragraph{Numerical results.} We compute strategies optimizing the importance of states from all models, demonstrating that attribution-based explanations can be computed for models with up to $10\,000$ of states and transitions.
Additionally, we highlight differences between the three presented encodings.

\Cref{tab:greps_bpic} compares the averaged computation times and results for Greps, BPIC12, and BPIC17.
For each model, we compute the lower and upper bounded importance for every state using all three encodings.
Encoding (Enc.) denotes the used encoding: \emph{\gqp} denotes the quadratic program optimizing over $\strategies_\mdp$, \emph{\qp} the quadratic encoding optimizing over reachability optimal strategies $\strategies^*_\mdp$, and \emph{\lp} the linear program.
Opt.\ denotes the number of problems solved optimally, T.O.\ denotes the timeouts, and Total the absolute number of problems.
The remaining columns describe the mean, standard deviation, minimum, and maximum over all solution times in seconds.
While \qp and \lp can solve all inputs without timeout, \lp solves all instances within one second.

\begin{table}[t]
    \centering
    \small
    \setlength{\tabcolsep}{2pt}
    \caption{Computing importance for GrepS and BPIC: averaged runtimes in seconds\label{tab:greps_bpic}}
    \vspace{-6pt}
    \begin{tabular}{llrrrrrrrr}
    \toprule
    \textbf{Model} & \textbf{Enc.} & \textbf{Mean(t)} & \textbf{Std(t)} & \textbf{Min(t)} & \textbf{Max(t)} & \textbf{Opt.} & \textbf{T.O.} & \textbf{Total} \\
    \midrule
    \multirow[t]{3}{*}{\textbf{Greps}} & \textbf{\gqp} & 0.06 & 0.06 & 0.00 & 0.64 & 128 & 0 & 128\\
    \textbf{} & \textbf{\qp} & 0.17 & 0.13 & 0.00 & 0.73 & 128 & 0 & 128\\
    \textbf{} & \textbf{\lp} & 0.03 & 0.03 & 0.00 & 0.19 & 128 & 0 & 128\\
    \multirow[t]{3}{*}{\textbf{BPIC12}} & \textbf{\gqp} & 1616.91 & 1419.14 & 0.02 & T.O. & 192 & 66 & 258 \\
    \textbf{} & \textbf{\qp} & 2.95 & 11.95 & 0.01 & 102.30 & 258 & 0 & 258 \\
    \textbf{} & \textbf{\lp} & 0.02 & 0.02 & 0.00 & 0.11 & 258 & 0 & 258 \\
    \multirow[t]{3}{*}{\textbf{BPIC17}} & \textbf{\gqp} & 23.28 & 209.62 & 0.03 & T.O. & 315 & 1 & 316 \\
    \textbf{} & \textbf{\qp} & 2.66 & 5.25 & 0.02 & 37.49 & 316 & 0 & 316 \\
    \textbf{} & \textbf{\lp} & 0.07 & 0.05 & 0.01 & 0.77 & 316 & 0 & 316 \\
    \bottomrule
    \end{tabular}
    \medskip
    \caption{Computing importance for Epidemic: averaged runtimes in seconds.\label{tab:epidemic}}
    \vspace{-6pt}
    \begin{tabular}{llrrrrrrrr}
    \toprule
    \textbf{Model} & \textbf{Enc.} & \textbf{Mean(t)} & \textbf{Std(t)} & \textbf{Min(t)} & \textbf{Max(t)} & \textbf{Opt.} & \textbf{T.O.} & \textbf{Total} \\
    \midrule
    \multirow[t]{2}{*}{\textbf{E6}} & \textbf{\qp} & 0.65 & 0.78 & 0.01 & 2.52 & 26 & 0 & 26 \\
    \textbf{} & \textbf{\lp} & 0.02 & 0.02 & 0.00 & 0.07 & 26 & 0 & 26 \\
    \multirow[t]{2}{*}{\textbf{E7}} & \textbf{\qp} & 2.28 & 3.17 & 0.01 & 12.10 & 30 & 0 & 30 \\
    \textbf{} & \textbf{\lp} & 0.01 & 0.00 & 0.00 & 0.01 & 30 & 0 & 30 \\
    \multirow[t]{2}{*}{\textbf{E8}} & \textbf{\qp} & 648.31 & 1387.09 & 0.01 & T.O. & 28 & 6 & 34 \\
    \textbf{} & \textbf{\lp} & 0.01 & 0.00 & 0.00 & 0.01 & 34 & 0 & 34 \\
    \multirow[t]{2}{*}{\textbf{E9}} & \textbf{\qp} & 687.61 & 1403.24 & 0.02 & T.O. & 31 & 7 & 38 \\
    \textbf{} & \textbf{\lp} & 0.01 & 0.00 & 0.00 & 0.01 & 38 & 0 & 38 \\
    \multirow[t]{2}{*}{\textbf{E10}} & \textbf{\qp} & 1399.82 & 1751.83 & 0.02 & T.O. & 26 & 16 & 42 \\
    \textbf{} & \textbf{\lp} & 0.01 & 0.00 & 0.01 & 0.02 & 42 & 0 & 42 \\
    \multirow[t]{2}{*}{\textbf{E15}} & \textbf{\qp} & 1742.02 & 1813.74 & 0.08 & T.O. & 32 & 30 & 62 \\
    \textbf{} & \textbf{\lp} & 0.04 & 0.02 & 0.01 & 0.08 & 62 & 0 & 62 \\
    \multirow[t]{2}{*}{\textbf{E19}} & \textbf{\qp} & 1753.99 & 1811.01 & 0.16 & T.O. & 40 & 38 & 78 \\
    \textbf{} & \textbf{\lp} & 0.05 & 0.03 & 0.02 & 0.11 & 78 & 0 & 78 \\
    \bottomrule
    \end{tabular}
\end{table} 

\Cref{tab:epidemic} shows runtimes for Epidemic with population sizes up to 19.
For population sizes larger than 19, the transition probabilities can become very small, i.e., they require more than 10 digits of precision to avoid rounding to $0$. 
Based on the results from BPIC12 and BPIC17, we omitted encoding \gqp as the Epidemic models are significantly larger.
Until E8, encoding \qp can be solved within the one-hour timeout, encoding \lp can be solved within fractions of a second for all models.
Comparing E7 with E8 reveals the growing complexity induced by the population size, the number of states increase from 962 to 1379, and transitions from 5644 to 9029.
All timed-out instances involved states with sufficient vaccinations for reaching $t$, and were well connected in the MDP.

\begin{figure}
\centering
\includegraphics[trim={10 10 10 10},clip, width=0.8\linewidth]{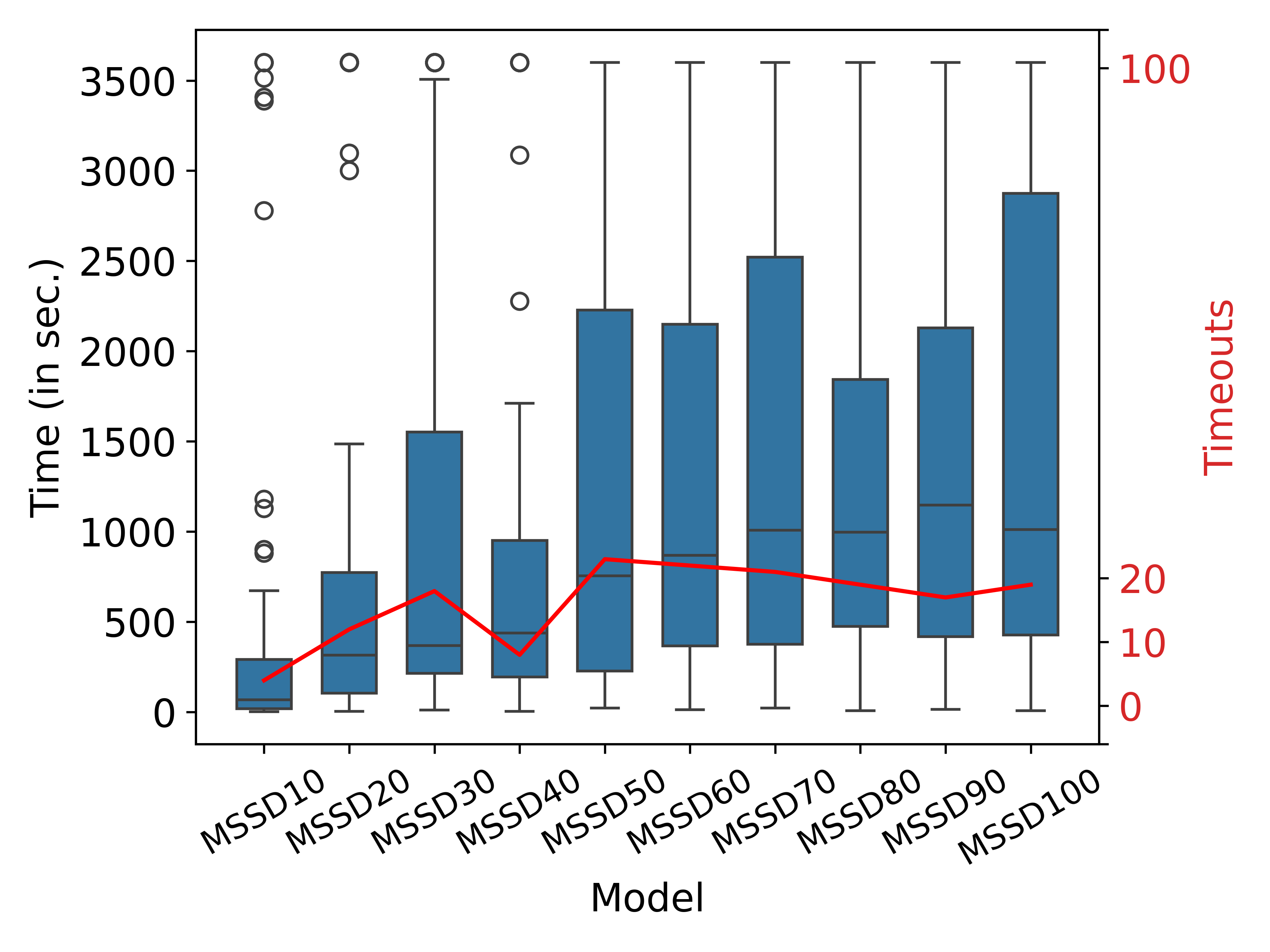}
\vspace{-6pt}
\caption{Summarized MSSD runtime results in seconds with timeouts for computing lower-bounded importance using encoding \lp.\label{fig:spotify_box}
}
\end{figure}

After confirming that computing strategies to minimize, respectively maximize, importance takes comparably long, we restricted the MSSD experiments to minimization.
\Cref{fig:spotify_box} presents the averaged runtime results for the MSSD models on the left-hand $y$-axis and the number of time-outs on the right-hand $y$-axis.
Most strategies were computed within a one-hour time constraint; the $4\unit{GB}$ memory constraint was reached in three MSSD10 problems.
While MSSD10--MSSD40 reported few timeouts, starting from MSSD50, up to $23$ experiments timed out.

\paragraph{Interpretability of our scores.}  Attribution-based explanations can support end-users in gaining insights into the decision process by highlighting the most crucial steps.
However, to be useful, attribution scores need to be understandable. Indeed, we observe that our technique is most informative when bounds on attribution scores are tight. In such cases, our attribution-based explanations lend themselves to a textual representation that summarizes  the most important states, respectively paths, with high lower bounds, and most detrimental states, respectively paths, with low upper bounds.
An excerpt of a textual representation of the most influential and detrimental states for receiving a loan in BPIC12 is:

{\begin{scriptsize}
\begin{verbatim}
States indispensable for a loan application include:
 A_CreateApplication, O_Accepted, and A_Pending.
States detrimental for a loan application include:
 SUPER LONG calls regarding incomplete information.
\end{verbatim}
\end{scriptsize}}

\noindent
Additionally, process representations where states and paths are colored with heat maps according to their importance can communicate attribution-based explanations well, as demonstrated in
\Cref{ex:gridworld} on a gridworld example.
\begin{example}
  \label{ex:gridworld}
  \Cref{fig:gridworld} shows a deterministic gridworld where the agent (red triangle) has to pick up a key, required to cross a river of deadly lava to reach the green target state.
  This gridworld can be modeled as an MDP where each state corresponds to an agent position, the target state is the state encoding the agent at the green tile position, and touching the lava leads to unsuccessful end states.
  The actions correspond to moving the agent in the 4-neighborhood, and picking up the key.
  The key is required to pass through the door, but no ``open'' action is needed.
  Without having the key, it is not possible to move on the door-field.

  Our importance scores enable answering questions such as:
  \emph{how important is visiting state $s$ for reaching the target?} (\textbf{P2}), or, 
  \emph{how important is following path $\tau$ 
  for reaching the target?} (\textbf{P4}).
  The importance of any state is depicted in \cref{fig:gridworld_overlay} via a heat map overlay.
  There exists a strategy that leads the agent via any non-lava state to the target state, and there exists no strategy that leads the agent via any lava state to the green state.
  However, there exists no successful strategy where the state across the lava receives an importance with a lower bound of less than $1$. 
\end{example}
\begin{figure}
  \centering
  \begin{subfigure}{0.49\linewidth}
    \centering
    \includegraphics[width=\linewidth]{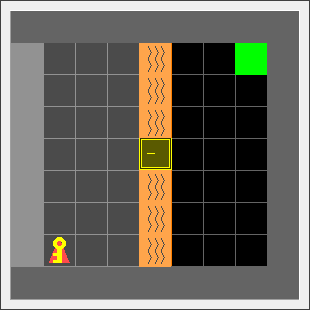}
    \caption{Gridworld Example.}
    \label{fig:gridworld}
  \end{subfigure}
  \begin{subfigure}{0.49\linewidth}
    \centering
    \includegraphics[width=\linewidth]{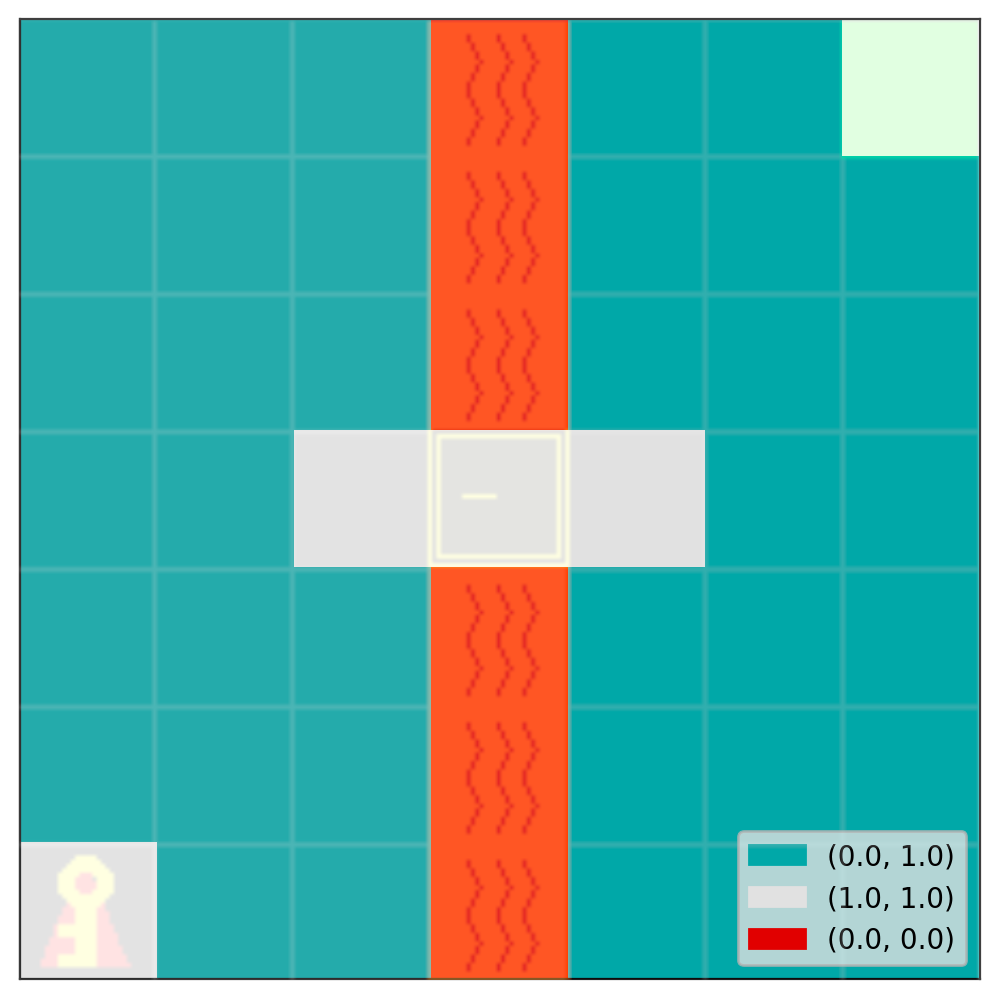}
    \caption{Gridworld Heat Map Overlay.}
    \label{fig:gridworld_overlay}
  \end{subfigure}
  \caption{Gridworld example, where the agent has to reach the green tile, while avoiding the lava tiles. 
  The heat map on the right visualizes the importance from the corresponding MDP and highlights that most states are optional.
  The three white colored transfer-states are indispensable, and the six red states should be avoided.}
  \label{fig:gridworld_example}
\end{figure}

\paragraph{Threats to validity.} The main threat to validity in our experiments is numerical
instability, related to the chosen constants for $M$ and $\epsilon$.
Due to its size, $M$ can introduce numerical inaccuracies as MDPs' transition probabilities might become very small,
spanning a large range of values in the model.  However, the
differences between the \lp and \qp encodings are small: on average,
results differ by less than $4 \cdot 10^{-4}$.  A table showing the
differences is included in Appendix~\ref{app:results}. In summary, attributive explanations can be computed efficiently for complex MDPs.
Notably, attributions are possible for strategies with optimal reachability in MDPs with up to $10\,000$ states and transitions. 






\section{Conclusion}
\label{sec:conclusion}

This paper introduces attribution-based explanations for Markov Decision Processes to address the limitations of static feature importance in sequential decision-making. We provide a formal characterization for attribution-based explanations in this context, including the assignment of importance scores to individual states as well as execution paths in the model. We also propose an optimization approach for computing these scores and demonstrate its scalability across a wide range of sequential decision-making problems.

Our work opens several interesting directions for future
work. First, we plan to investigate extensions of our framework towards computing importance of subpaths as opposed the full paths (from initial to goal states) considered here. To this aim, we will investigate a formal grammar capable of identifying semantically interesting segments within an execution. Second, 
we intend to explore the application of our techniques to Deep Reinforcement Learning. This will require novel abstraction methods to map latent representations learned by neural networks onto our formal attribution framework, bridging the gap between symbolic reasoning and neural decision-making. Finally, we believe it would be important to conduct user studies to evaluate the interpretability of our attribution scheme and assess how effectively these explanations help end users understand system behavior.



\bibliographystyle{named}
\bibliography{references}

\clearpage
\appendix

\section{Proofs for Theoretical Properties}

\subsection{Monotonicity}
\label{app:monotonicity}
\begin{lemma}
  \label{lem:app_greedily}
  Optimizing the importance $\impSigma{s}$ is monotonic in the reachability probability $\Pr_{\mdp^\strategy}(s)$, but is not monotonic in $\Pr_{\mdp^\strategy}(s')$ for all $s' \in \states\setminus\{s\}$. 
\end{lemma}
\begin{proof}
  The monotonicity of $\impSigma{s}$ in $\Pr_{\mdp^\strategy}(s)$ for $s \in \states \setminus \{s\}$ follows from the monotonicity of the derivative of $\impSigma{s}$.
 \Cref{ex:abs-relevance} demonstrates that $\impSigma{s}$ is not monotonic in $\Pr_{\mdp^\strategy}(s')$ for $s' \neq s$.
\end{proof}

\begin{figure}
  \centering
 \begin{subfigure}{\linewidth}  
   \centering
   \begin{tikzpicture}[auto,node distance=8mm,>=latex,font=\small]
  \def\centerarc[#1](#2)(#3:#4:#5)
  { \draw[#1] ($(#2)+({#5*cos(#3)},{#5*sin(#3)})$) arc (#3:#4:#5); }

  \tikzstyle{round}=[thick,draw=black,circle]

  \node[round] (s0) {$\sinit$};
  \node[round, right=4em of s0] (s1) {$s_1$};
  \node[round, right=4em of s1] (s2) {$s_2$};

  \node[round,above left=3em and 0em of s1] (tau2) {$\tau$};
  \node[round,above right=3em and 1.5em of s1] (st) {$s_t$};

  \draw[->] (s0) -- (s1) node[midway, below]{0.1};
  \draw[->] (s0) edge[bend right=40] node[midway, near start, below] {0.9} (s2);
  \centerarc[postaction={decorate, decoration={text align={center},raise={1mm}, text along path, text={a}}}](s0)(15:-55:0.5) node[midway];

  \draw[->] (s1) -- (st) node[midway, near start, right]{0.1};
  \draw[->] (s1) -- (tau2) node[midway, left]{0.9};
  \centerarc[postaction={decorate, decoration={text align={center},raise={1mm},text along path, text={a}}}](s1)(120:45:0.5) node[midway];

  \draw[->] (s2) -- (st) node[midway, near start, right]{1};
  \draw[->] (s2) edge[dashed] node[midway, near start, below, dashed]{1} (s1);

\end{tikzpicture}
   \subcaption{Example MDP $\mdp_\mathrm{ex1}$.}
   \label{fig:mdp_example}
 \end{subfigure}

   \begin{subfigure}{0.48\linewidth}
     \centering
     \begin{tabular}{|c|c|c|}\hline
       \diagbox{$\states$}{$\actions$} & a & b \\\hline
       $\sinit$ & 1 & 0 \\
       $s_1$ & 1 & 0 \\
       $s_2$ & 1 & 0 \\\hline
     \end{tabular}
     \subcaption{Initial strategy $\sigma_0$ for $\mdp_\mathrm{ex1}$.}
     \label{fig:initial_strat}
   \end{subfigure}
   \begin{subfigure}{0.48\linewidth}
     \centering
     \begin{tabular}{|c|c|c|}\hline
       \diagbox{$\states$}{$\actions$} & a & b \\\hline
       $\sinit$ & 1 & 0 \\
       $s_1$ & 1 & 0 \\
       $s_2$ & 0 & \bf 1 \\\hline
     \end{tabular}
     \subcaption{Updated strategy $\sigma_1$ for $s_2$.}
     \label{fig:strat_s0}
   \end{subfigure}
 \caption{Example for Non-monotone Sequence of Improvements.}
\end{figure}
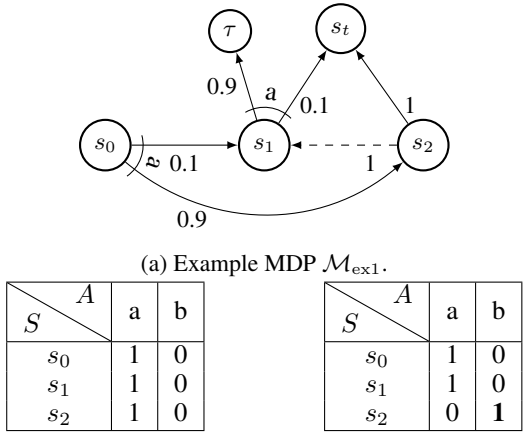

\begin{example}\label{ex:abs-relevance}
  For optimizing reachability probabilities, locally optimal decisions
  are globally optimal, so optimal strategies can be computed greedily
  over individual reachability
  probabilities~\cite{forejt2011automated}.  Strategies optimizing the
  importance of individual states cannot be computed this way.

  Consider the MDP $\mdp_\mathrm{ex1}$ (see \cref{fig:mdp_example}).
  To compute a strategy $\sigma$ to maximize
  $\imp_{\mdp_{\mathrm{ex1}}^{\sigma}}^{\sinit, t}(s_1)$, it is not
  sufficient to consider local decisions to maximize the reachability
  probabilities.  The initial strategy $\strategy_0$ in
  \cref{fig:initial_strat} maximizes the probability of reaching $t$
  in every state.  The importance of $s_1$ under $\strategy_0$ is
  $\imp_{\mdp_{\mathrm{ex1}}^{\strategy_{0}}}^{s_0, t}(s_1) = 0.01$,
  and that of $s_2$ is
  $\imp_{\mdp_{\mathrm{ex1}}^{\strategy_{0}}}^{s_0, t}(s_2) = 0.99$.
  For the updated strategy $\strategy_1$ in \cref{fig:strat_s0}, the
  importance of $s_1$ becomes
  $\imp_{\mdp_{\mathrm{ex1}}^{\strategy_{1}}}^{s_0, t}(s_1) = 1$,
  while the probability of reaching $t$ from $s_2$ reduces from $1$
  to $0.1$, and the importance of $s_2$ to
  $\imp_{\mdp_{\mathrm{ex1}}^{\strategy_{1}}}^{s_0, t}(s_2) = 0.9$.
  Note how the strategy to maximize $\imp(s_2)$ is not (strictly)
  monotonic in either the importance or the reachability values of
  $\sinit$ or $s_1$.
\end{example}


\subsection{Encoding Properties}
\label{app:encoding_properties}
In the sequel, we let $P$ denote the optimization problem defined by Equations~\ref{eq:qp_target}--\ref{eq:qp_path}.
\begin{lemma}
  \label{lem:ap_optimal}
  Assume a solution $\nu$ to $P$, which maps each variable $v$ to value $\nu(v)$. 
  Further, let $\strategy$ be the strategy defined by $\strategy(s,a) = \nu(p_{sa})$, for all $\state \in \states$ and $\action \in \actions$.
  Then, the value of the objective function equals $\min_{\sigma \in \Sigma_\mdp^t} \imp_{\mdp^{\sigma}}^{\sinit, t}(\hat{s})$.
\end{lemma}
\begin{proofsketch}
For this proof, we assume that $\epsilon$ is chosen small enough to not exclude any strategies.
Further, observe that $P$ computes the reachability for each state $s$ of reaching $t$ by visiting $\hat{s}$ and avoiding $\hat{s}$.
Constraint~\ref{eq:qp_bellman} computes the reachability of a state over all successor states, while Constraint~\ref{eq:qp_path} ensures that a terminal state is actually reached.
There exists a weight-decreasing assignment for all states satisfying Constraint~\ref{eq:qp_path} iff the corresponding strategy allows for a simple path $\tau$ reaching a terminal state with positive probability from every state.
We prove the claim via induction over the length of the shortest path reaching a terminal state from each state and the corresponding weight assignment.
With a length of $0$, the states are terminal states and can be assigned a very large constant, e.g., a multiple of the product over all transition probabilities.
For a length of $i>0$, all states reaching a terminal state in $i-1$-many steps allow for a valid assignment, where a possible assignment for $i$-many steps needs to be shown.
Let  $s \in \states$ be a state reaching a terminal state in $i$-many steps.
State $s$ reaches at least one state $s'$ which was covered in step $i-1$ and, thus, $\nu(\tau_{s'})$ is well-defined.
The weight for $s$ can now be set to any value satisfying $\tau_s + 1 \leq \sum_{a \in \actions(s)} \transitions(s,a,s') \cdot \tau_{s'}$.
Constraint~\ref{eq:qp_lower_bound} ensures that $P$ considers all strategies that reach $t$ with a positive reachability probability.
\end{proofsketch}

\begin{theorem}[Soundness and Completeness]
    $P$ is feasible
    iff the importance of state $\hat{s}$ is well defined.
\end{theorem}
\begin{proofsketch}
$\rightarrow$.
Let $\nu$ be a solution to $P$. 
The strategy $\strategy$ optimizing the importance is induced by variables $p_{sa}$, $\strategy(s)(a) = \nu(p_{sa})$ for all $s \in \states$ and $a \in \actions$.
Following \cref{lem:optimal}, this value is the importance of $\hat{s}$.

\noindent $\leftarrow$.
Let $x = \min_{\strategy \in \strategies^t_\mdp}\impSigma{\hat{s}}$ be the minimal importance of state $\hat{s}$, realized by strategy $\strategy$.
This strategy can be extended to a solution of $P$ by (1) instantiating variables $p_{sa}$ with $\strategy(s)(a)$, (2) solving the Bellman equations in Constraint~\ref{eq:qp_bellman} and storing the reachability values in $p_{s, s^{\bowtie}_t}$, and (3) computing state orderings for $\tau_s$ following the construction in \cref{lem:optimal}.
\end{proofsketch}

The following lemma connects the feasibility of the encoding to the reachability of the target state $t$.

\begin{lemma}
    The importance of state $\hat{s}$ is well defined iff $\strategies^t_\mdp~\neq~\emptyset$.
\end{lemma}
\begin{proof}
If there exists no strategy under which $t$ is reached, thus $\strategies^t_\mdp = \emptyset$, then the importance of $\hat{s}$, $\impSigma{\hat{s}}$, is not well defined for any $\strategy \in \strategies_\mdp$; the denominator always defaults to $0$.
However, if $\strategies^t_\mdp \neq \emptyset$, there exists a strategy $\strategy \in \strategies^t_\mdp$ under which $t$ is reached with non-zero probability, and the importance of state $\hat{s}$ is thus well defined.
\end{proof}

\subsection{Deterministic Strategies}
\label{app:deterministic_strategy}
\begin{lemma}
    \label{lem:app_pure_strat}
    For an MDP $\mdp$, target state $t$, state $s \in \states$ and simple path $\tau$, there exists a deterministic strategy optimizing $\imp_{\mdp^{\sigma}}^{\sinit, t}(s)$ and $\imp_{\mdp^{\sigma}}^{\sinit, t}(\tau)$ from the strategies in $\strategies^*_\mdp$.
\end{lemma}
\begin{proof}
  Let $\mdp$ be an MDP, state $t \in \states$ be the target state, $s \in \states$ a state in $\mdp$, and $\tau =  \sinit, a_0, \dots, s_n$ a simple path.
  If $\strategies^t_\mdp \neq \emptyset$, there exists a deterministic strategy $\strategy$ in $\strategies^*_\mdp$~\cite{baier2008principles}.
  Thus, to optimize the probability of visiting $s$ among the strategies in $\strategies^*_\mdp$, it suffices to consider deterministic strategies.
  Further, to optimize the probability of following $\tau$, the deterministic strategy that optimizes that probability among $\strategies^*_\mdp$ is selected.
\end{proof}

\subsection{Default Values}
\begin{lemma}\label{lemma:app_rel}
  For all strategies $\sigma \in \Sigma^t$ holds $\imp_{\mdp^{\sigma}}^{\sinit, t}(\sinit) = \imp_{\mdp^{\sigma}}^{\sinit, t}(t) = 1$.
\end{lemma}
\begin{proof}
  \Cref{lemma:app_rel} follows from unfolding $\Pr_{\mdp^\sigma}(\cond{s})$ following LTL definitions~\cite{pnueli1977temporal}:
  $\Pr_{\mdp^\sigma}(\cond{\sinit}) = P_{\mdp^\sigma}(\lozenge t \land (\lnot t \U \sinit)) = P_{\mdp^\sigma}(\lozenge t)$, holds as $\sinit$ is the initial state of $\mdp$, and consequently, $\imp_{\mdp^{\sigma}}^{\sinit, t}(\sinit) = 1$. 

  $P_{\mdp^\sigma}(\cond{t}) = 
  P_{\mdp^\sigma}(\lnot t \U t)$. As $t$ is a unique state, $\lnot t \U t$ reduces to $\lozenge t$ and thus $\imp_{\mdp^{\sigma}}^{\sinit, t}(t) = 1$.
\end{proof}

\section{Additional Results}
We here expand on the results of the experiments reported in the paper, i.e., the different runtimes for minimization and maximization in \Cref{tab:app_greps_bpic_seperated}, as well as all results for Epidemic, \Cref{tab:app_epidemic} and \Cref{tab:app_epidemic_seperated}.
\Cref{tab:greps_bpic_differences} compares the differences between the importance computed with \qp and with \lp, highlighting that up to solver accuracy, both encodings compute comparable results.
We elaborate on the solver accuracy in our threats to validity.

\subsection{Additional Experimental Results}
\label{app:results}
\begin{table}[h]
    \centering
    \small
    \setlength{\tabcolsep}{2pt}
    \begin{tabular}{llrrrrrrrr}
    \toprule
    \textbf{Model} & \textbf{Enc.} & \textbf{Sense} & \textbf{Mean(t)} & \textbf{Std(t)} & \textbf{Min(t)} & \textbf{Max(t)} & \textbf{Opt.} & \textbf{T.O.} \\
    \midrule
    \multirow[t]{6}{*}{\textbf{Greps}} & \multirow[t]{2}{*}{\textbf{\gqp}} & \textbf{min} & 0.07 & 0.08 & 0.00 & 0.64 & 64 & 0 \\
    \textbf{} & \textbf{} & \textbf{max} & 0.04 & 0.03 & 0.00 & 0.14 & 64 & 0 \\
    \textbf{} & \multirow[t]{2}{*}{\textbf{\qp}} & \textbf{min} & 0.21 & 0.13 & 0.00 & 0.73 & 64 & 0 \\
    \textbf{} & \textbf{} & \textbf{max} & 0.13 & 0.12 & 0.00 & 0.70 & 64 & 0 \\
    \textbf{} & \multirow[t]{2}{*}{\textbf{\lp}} & \textbf{min} & 0.04 & 0.03 & 0.00 & 0.19 & 64 & 0 \\
    \textbf{} & \textbf{} & \textbf{max} & 0.03 & 0.02 & 0.00 & 0.10 & 64 & 0 \\
    \multirow[t]{6}{*}{\textbf{BPIC12}} & \multirow[t]{2}{*}{\textbf{\gqp}} & \textbf{min} & 930.11 & 1081.63 & 0.02 & T.O. & 119 & 10 \\
    \textbf{} & \textbf{} & \textbf{max} & 2303.72 & 1386.69 & 0.02 & T.O. & 73 & 56 \\
    \textbf{} & \multirow[t]{2}{*}{\textbf{\qp}} & \textbf{min} & 5.51 & 16.51 & 0.01 & 102.30 & 129 & 0 \\
    \textbf{} & \textbf{} & \textbf{max} & 0.39 & 0.97 & 0.01 & 9.71 & 129 & 0 \\
    \textbf{} & \multirow[t]{2}{*}{\textbf{\lp}} & \textbf{min} & 0.03 & 0.02 & 0.01 & 0.11 & 129 & 0 \\
    \textbf{} & \textbf{} & \textbf{max} & 0.02 & 0.00 & 0.00 & 0.03 & 129 & 0 \\
    \multirow[t]{6}{*}{\textbf{BPIC17}} & \multirow[t]{2}{*}{\textbf{\gqp}} & \textbf{min} & 5.16 & 12.46 & 0.03 & 65.05 & 158 & 0 \\
    \textbf{} & \textbf{} & \textbf{max} & 41.39 & 295.54 & 0.03 & T.O. & 157 & 1 \\
    \textbf{} & \multirow[t]{2}{*}{\textbf{\qp}} & \textbf{min} & 2.34 & 5.01 & 0.02 & 37.49 & 158 & 0 \\
    \textbf{} & \textbf{} & \textbf{max} & 2.99 & 5.48 & 0.03 & 36.49 & 158 & 0 \\
    \textbf{} & \multirow[t]{2}{*}{\textbf{\lp}} & \textbf{min} & 0.06 & 0.04 & 0.01 & 0.19 & 158 & 0 \\
    \textbf{} & \textbf{} & \textbf{max} & 0.08 & 0.06 & 0.01 & 0.77 & 158 & 0 \\
    \bottomrule
    \end{tabular}
    \caption{Averaged GrepS and BPIC runtime results in seconds for computing importance.}
    \label{tab:app_greps_bpic_seperated}
\end{table} 
\begin{table}
    \centering
    \small
    \setlength{\tabcolsep}{2pt}
    \begin{tabular}{llrrrr}
    \toprule
    \textbf{Model} & \textbf{Sense} & \textbf{Mean(t)} & \textbf{Std(t)} & \textbf{Min(t)} & \textbf{Max(t)}\\
    \midrule
    \multirow[t]{2}{*}{\textbf{Greps}} & \textbf{min} & 1.16e-16 & 2.94e-16 & 0 & 1.33e-15 \\
    \textbf{} & \textbf{max} & 1.10e-16 & 3.05e-16 & 0 & 1.67e-15 \\
    \multirow[t]{2}{*}{\textbf{BPIC12}} & \textbf{min} & 0 & 0 & 0 & 0 \\
    \textbf{} & \textbf{max} & 4.27e-08 & 4.85e-07 & 0 & 5.51e-06 \\
    \multirow[t]{2}{*}{\textbf{BPIC17}} & \textbf{min} & 3.88e-04 & 2.81e-03 & 0 & 3.33e-02 \\
    \textbf{} & \textbf{max} & 7.78e-05 & 3.32e-04 & 0 & 3.69e-03 \\
    \bottomrule
    \end{tabular}
    \caption{Averaged absolute differences between \lp and \qp.}
    \label{tab:greps_bpic_differences}
\end{table}

\begin{table}
    \centering
    \small
    \setlength{\tabcolsep}{2pt}
    \begin{tabular}{llrrrrrrrr}
    \toprule
    \textbf{Model} & \textbf{Enc.} & \textbf{Mean(t)} & \textbf{Std(t)} & \textbf{Min(t)} & \textbf{Max(t)} & \textbf{Opt.} & \textbf{T.O.} & \textbf{M.O.} \\
    \midrule
    \multirow[t]{2}{*}{\textbf{E4}} & \textbf{\qp} & 0.09 & 0.09 & 0.01 & 0.26 & 18 & 0 & 0 \\
    \textbf{} & \textbf{\lp} & 0.01 & 0.01 & 0.00 & 0.03 & 18 & 0 & 0 \\
    \hline
    \multirow[t]{2}{*}{\textbf{E5}} & \textbf{\qp} & 0.19 & 0.21 & 0.01 & 0.53 & 22 & 0 & 0 \\
    \textbf{} & \textbf{\lp} & 0.02 & 0.04 & 0.00 & 0.19 & 22 & 0 & 0 \\
    \hline
    \multirow[t]{2}{*}{\textbf{E6}} & \textbf{\qp} & 0.65 & 0.78 & 0.01 & 2.52 & 26 & 0 & 0 \\
    \textbf{} & \textbf{\lp} & 0.02 & 0.02 & 0.00 & 0.07 & 26 & 0 & 0 \\
    \hline
    \multirow[t]{2}{*}{\textbf{E7}} & \textbf{\qp} & 2.28 & 3.17 & 0.01 & 12.10 & 30 & 0 & 0 \\
    \textbf{} & \textbf{\lp} & 0.01 & 0.00 & 0.00 & 0.01 & 30 & 0 & 0 \\
    \hline
    \multirow[t]{2}{*}{\textbf{E8}} & \textbf{\qp} & 648.31 & 1387.09 & 0.01 & T.O. & 28 & 6 & 0 \\
    \textbf{} & \textbf{\lp} & 0.01 & 0.00 & 0.00 & 0.01 & 34 & 0 & 0 \\
    \hline
    \multirow[t]{2}{*}{\textbf{E9}} & \textbf{\qp} & 687.61 & 1403.24 & 0.02 & T.O. & 31 & 7 & 0 \\
    \textbf{} & \textbf{\lp} & 0.01 & 0.00 & 0.00 & 0.01 & 38 & 0 & 0 \\
    \hline
    \multirow[t]{2}{*}{\textbf{E10}} & \textbf{\qp} & 1399.82 & 1751.83 & 0.02 & T.O. & 26 & 16 & 0 \\
    \textbf{} & \textbf{\lp} & 0.01 & 0.00 & 0.01 & 0.02 & 42 & 0 & 0 \\
    \hline
    \multirow[t]{2}{*}{\textbf{E11}} & \textbf{\qp} & 1654.71 & 1804.20 & 0.03 & T.O. & 25 & 21 & 0 \\
    \textbf{} & \textbf{\lp} & 0.01 & 0.01 & 0.01 & 0.02 & 46 & 0 & 0 \\
    \hline
    \multirow[t]{2}{*}{\textbf{E12}} & \textbf{\qp} & 1728.06 & 1816.83 & 0.04 & T.O. & 26 & 24 & 0 \\
    \textbf{} & \textbf{\lp} & 0.02 & 0.01 & 0.01 & 0.03 & 50 & 0 & 0 \\
    \hline
    \multirow[t]{2}{*}{\textbf{E13}} & \textbf{\qp} & 1733.40 & 1815.66 & 0.05 & T.O. & 28 & 26 & 0 \\
    \textbf{} & \textbf{\lp} & 0.02 & 0.01 & 0.01 & 0.04 & 54 & 0 & 0 \\
    \hline
    \multirow[t]{2}{*}{\textbf{E14}} & \textbf{\qp} & 1738.01 & 1814.65 & 0.06 & T.O. & 30 & 28 & 0 \\
    \textbf{} & \textbf{\lp} & 0.03 & 0.02 & 0.01 & 0.06 & 58 & 0 & 0 \\
    \hline
    \multirow[t]{2}{*}{\textbf{E15}} & \textbf{\qp} & 1742.02 & 1813.74 & 0.08 & T.O. & 32 & 30 & 0 \\
    \textbf{} & \textbf{\lp} & 0.04 & 0.02 & 0.01 & 0.08 & 62 & 0 & 0 \\
    \hline
    \multirow[t]{2}{*}{\textbf{E16}} & \textbf{\qp} & 1745.57 & 1812.96 & 0.10 & T.O. & 34 & 32 & 0 \\
    \textbf{} & \textbf{\lp} & 0.06 & 0.03 & 0.01 & 0.14 & 66 & 0 & 0 \\
    \hline
    \multirow[t]{2}{*}{\textbf{E17}} & \textbf{\qp} & 1748.70 & 1812.25 & 0.12 & T.O. & 36 & 34 & 0 \\
    \textbf{} & \textbf{\lp} & 0.04 & 0.02 & 0.01 & 0.09 & 70 & 0 & 0 \\
    \hline
    \multirow[t]{2}{*}{\textbf{E18}} & \textbf{\qp} & 1751.51 & 1811.62 & 0.14 & T.O. & 38 & 36 & 0 \\
    \textbf{} & \textbf{\lp} & 0.04 & 0.02 & 0.02 & 0.09 & 74 & 0 & 0 \\
    \hline
    \multirow[t]{2}{*}{\textbf{E19}} & \textbf{\qp} & 1753.99 & 1811.01 & 0.16 & T.O. & 40 & 38 & 0 \\
    \textbf{} & \textbf{\lp} & 0.05 & 0.03 & 0.02 & 0.11 & 78 & 0 & 0 \\
    \bottomrule
    \end{tabular}
    \caption{Averaged Epidemic runtime results in seconds for computing importance.}
    \label{tab:app_epidemic}
\end{table} 

\begin{table*}
    \centering
    \small
    \setlength{\tabcolsep}{2pt}
    \begin{tabular}{llrrrrrrrr}
    \toprule
    \textbf{Model} & \textbf{Enc.} & \textbf{Sense} & \textbf{Mean(t)} & \textbf{Std(t)} & \textbf{Min(t)} & \textbf{Max(t)} & \textbf{Opt.} & \textbf{T.O.} & \textbf{M.O.} \\
    \midrule
    \multirow[c]{4}{*}{\textbf{Epidemic4}} & \multirow[c]{2}{*}{\textbf{\qp}} & \textbf{min} & 0.09 & 0.10 & 0.01 & 0.23 & 9 & 0 & 0 \\
    \textbf{} & \textbf{} & \textbf{max} & 0.08 & 0.08 & 0.01 & 0.26 & 9 & 0 & 0 \\
    \cline{2-10}
    \textbf{} & \multirow[c]{2}{*}{\textbf{\lp}} & \textbf{min} & 0.01 & 0.01 & 0.00 & 0.03 & 9 & 0 & 0 \\
    \textbf{} & \textbf{} & \textbf{max} & 0.01 & 0.01 & 0.00 & 0.02 & 9 & 0 & 0 \\
    \hline
    \multirow[c]{4}{*}{\textbf{Epidemic5}} & \multirow[c]{2}{*}{\textbf{\qp}} & \textbf{min} & 0.16 & 0.17 & 0.01 & 0.43 & 11 & 0 & 0 \\
    \textbf{} & \textbf{} & \textbf{max} & 0.23 & 0.24 & 0.01 & 0.53 & 11 & 0 & 0 \\
    \cline{2-10}
    \textbf{} & \multirow[c]{2}{*}{\textbf{\lp}} & \textbf{min} & 0.01 & 0.01 & 0.00 & 0.02 & 11 & 0 & 0 \\
    \textbf{} & \textbf{} & \textbf{max} & 0.03 & 0.05 & 0.00 & 0.19 & 11 & 0 & 0 \\
    \hline
    \multirow[c]{4}{*}{\textbf{Epidemic6}} & \multirow[c]{2}{*}{\textbf{\qp}} & \textbf{min} & 0.82 & 0.99 & 0.01 & 2.52 & 13 & 0 & 0 \\
    \textbf{} & \textbf{} & \textbf{max} & 0.48 & 0.48 & 0.02 & 1.18 & 13 & 0 & 0 \\
    \cline{2-10}
    \textbf{} & \multirow[c]{2}{*}{\textbf{\lp}} & \textbf{min} & 0.02 & 0.02 & 0.00 & 0.06 & 13 & 0 & 0 \\
    \textbf{} & \textbf{} & \textbf{max} & 0.03 & 0.03 & 0.01 & 0.07 & 13 & 0 & 0 \\
    \hline
    \multirow[c]{4}{*}{\textbf{Epidemic7}} & \multirow[c]{2}{*}{\textbf{\qp}} & \textbf{min} & 3.02 & 4.08 & 0.03 & 12.10 & 15 & 0 & 0 \\
    \textbf{} & \textbf{} & \textbf{max} & 1.54 & 1.72 & 0.01 & 3.78 & 15 & 0 & 0 \\
    \cline{2-10}
    \textbf{} & \multirow[c]{2}{*}{\textbf{\lp}} & \textbf{min} & 0.00 & 0.00 & 0.00 & 0.01 & 15 & 0 & 0 \\
    \textbf{} & \textbf{} & \textbf{max} & 0.01 & 0.00 & 0.01 & 0.01 & 15 & 0 & 0 \\
    \hline
    \multirow[c]{4}{*}{\textbf{Epidemic8}} & \multirow[c]{2}{*}{\textbf{\qp}} & \textbf{min} & 437.43 & 1190.54 & 0.02 & T.O. & 15 & 2 & 0 \\
    \textbf{} & \textbf{} & \textbf{max} & 859.18 & 1567.29 & 0.01 & T.O. & 13 & 4 & 0 \\
    \cline{2-10}
    \textbf{} & \multirow[c]{2}{*}{\textbf{\lp}} & \textbf{min} & 0.01 & 0.00 & 0.00 & 0.01 & 17 & 0 & 0 \\
    \textbf{} & \textbf{} & \textbf{max} & 0.01 & 0.00 & 0.00 & 0.01 & 17 & 0 & 0 \\
    \hline
    \multirow[c]{4}{*}{\textbf{Epidemic9}} & \multirow[c]{2}{*}{\textbf{\qp}} & \textbf{min} & 597.52 & 1336.78 & 0.02 & T.O. & 16 & 3 & 0 \\
    \textbf{} & \textbf{} & \textbf{max} & 777.69 & 1497.81 & 0.02 & T.O. & 15 & 4 & 0 \\
    \cline{2-10}
    \textbf{} & \multirow[c]{2}{*}{\textbf{\lp}} & \textbf{min} & 0.01 & 0.00 & 0.00 & 0.01 & 19 & 0 & 0 \\
    \textbf{} & \textbf{} & \textbf{max} & 0.01 & 0.00 & 0.00 & 0.01 & 19 & 0 & 0 \\
    \hline
    \multirow[c]{4}{*}{\textbf{Epidemic10}} & \multirow[c]{2}{*}{\textbf{\qp}} & \textbf{min} & 1418.69 & 1762.74 & 0.02 & T.O. & 13 & 8 & 0 \\
    \textbf{} & \textbf{} & \textbf{max} & 1380.96 & 1784.16 & 0.02 & T.O. & 13 & 8 & 0 \\
    \cline{2-10}
    \textbf{} & \multirow[c]{2}{*}{\textbf{\lp}} & \textbf{min} & 0.01 & 0.01 & 0.01 & 0.02 & 21 & 0 & 0 \\
    \textbf{} & \textbf{} & \textbf{max} & 0.01 & 0.00 & 0.01 & 0.02 & 21 & 0 & 0 \\
    \hline
    \multirow[c]{4}{*}{\textbf{Epidemic11}} & \multirow[c]{2}{*}{\textbf{\qp}} & \textbf{min} & 1587.64 & 1807.74 & 0.03 & T.O. & 13 & 10 & 0 \\
    \textbf{} & \textbf{} & \textbf{max} & 1721.78 & 1838.72 & 0.03 & T.O. & 12 & 11 & 0 \\
    \cline{2-10}
    \textbf{} & \multirow[c]{2}{*}{\textbf{\lp}} & \textbf{min} & 0.01 & 0.01 & 0.01 & 0.02 & 23 & 0 & 0 \\
    \textbf{} & \textbf{} & \textbf{max} & 0.01 & 0.00 & 0.01 & 0.02 & 23 & 0 & 0 \\
    \hline
    \multirow[c]{4}{*}{\textbf{Epidemic12}} & \multirow[c]{2}{*}{\textbf{\qp}} & \textbf{min} & 1728.06 & 1835.66 & 0.04 & T.O. & 13 & 12 & 0 \\
    \textbf{} & \textbf{} & \textbf{max} & 1728.06 & 1835.66 & 0.04 & T.O. & 13 & 12 & 0 \\
    \cline{2-10}
    \textbf{} & \multirow[c]{2}{*}{\textbf{\lp}} & \textbf{min} & 0.01 & 0.00 & 0.01 & 0.03 & 25 & 0 & 0 \\
    \textbf{} & \textbf{} & \textbf{max} & 0.02 & 0.00 & 0.01 & 0.03 & 25 & 0 & 0 \\
    \hline
    \multirow[c]{4}{*}{\textbf{Epidemic13}} & \multirow[c]{2}{*}{\textbf{\qp}} & \textbf{min} & 1733.40 & 1833.04 & 0.05 & T.O. & 14 & 13 & 0 \\
    \textbf{} & \textbf{} & \textbf{max} & 1733.39 & 1833.03 & 0.05 & T.O. & 14 & 13 & 0 \\
    \cline{2-10}
    \textbf{} & \multirow[c]{2}{*}{\textbf{\lp}} & \textbf{min} & 0.01 & 0.01 & 0.01 & 0.03 & 27 & 0 & 0 \\
    \textbf{} & \textbf{} & \textbf{max} & 0.03 & 0.00 & 0.02 & 0.04 & 27 & 0 & 0 \\
    \hline
    \multirow[c]{4}{*}{\textbf{Epidemic14}} & \multirow[c]{2}{*}{\textbf{\qp}} & \textbf{min} & 1738.01 & 1830.78 & 0.06 & T.O. & 15 & 14 & 0 \\
    \textbf{} & \textbf{} & \textbf{max} & 1738.01 & 1830.77 & 0.06 & T.O. & 15 & 14 & 0 \\
    \cline{2-10}
    \textbf{} & \multirow[c]{2}{*}{\textbf{\lp}} & \textbf{min} & 0.01 & 0.01 & 0.01 & 0.03 & 29 & 0 & 0 \\
    \textbf{} & \textbf{} & \textbf{max} & 0.04 & 0.01 & 0.03 & 0.06 & 29 & 0 & 0 \\
    \hline
    \multirow[c]{4}{*}{\textbf{Epidemic15}} & \multirow[c]{2}{*}{\textbf{\qp}} & \textbf{min} & 1742.01 & 1828.79 & 0.08 & T.O. & 16 & 15 & 0 \\
    \textbf{} & \textbf{} & \textbf{max} & 1742.02 & 1828.79 & 0.08 & T.O. & 16 & 15 & 0 \\
    \cline{2-10}
    \textbf{} & \multirow[c]{2}{*}{\textbf{\lp}} & \textbf{min} & 0.04 & 0.03 & 0.01 & 0.08 & 31 & 0 & 0 \\
    \textbf{} & \textbf{} & \textbf{max} & 0.05 & 0.01 & 0.02 & 0.08 & 31 & 0 & 0 \\
    \hline
    \multirow[c]{4}{*}{\textbf{Epidemic16}} & \multirow[c]{2}{*}{\textbf{\qp}} & \textbf{min} & 1745.59 & 1827.08 & 0.10 & T.O. & 17 & 16 & 0 \\
    \textbf{} & \textbf{} & \textbf{max} & 1745.56 & 1827.06 & 0.10 & T.O. & 17 & 16 & 0 \\
    \cline{2-10}
    \textbf{} & \multirow[c]{2}{*}{\textbf{\lp}} & \textbf{min} & 0.06 & 0.03 & 0.01 & 0.14 & 33 & 0 & 0 \\
    \textbf{} & \textbf{} & \textbf{max} & 0.06 & 0.02 & 0.02 & 0.09 & 33 & 0 & 0 \\
    \hline
    \multirow[c]{4}{*}{\textbf{Epidemic17}} & \multirow[c]{2}{*}{\textbf{\qp}} & \textbf{min} & 1748.70 & 1825.53 & 0.12 & T.O. & 18 & 17 & 0 \\
    \textbf{} & \textbf{} & \textbf{max} & 1748.70 & 1825.53 & 0.12 & T.O. & 18 & 17 & 0 \\
    \cline{2-10}
    \textbf{} & \multirow[c]{2}{*}{\textbf{\lp}} & \textbf{min} & 0.03 & 0.02 & 0.01 & 0.07 & 35 & 0 & 0 \\
    \textbf{} & \textbf{} & \textbf{max} & 0.06 & 0.01 & 0.05 & 0.09 & 35 & 0 & 0 \\
    \hline
    \multirow[c]{4}{*}{\textbf{Epidemic18}} & \multirow[c]{2}{*}{\textbf{\qp}} & \textbf{min} & 1751.51 & 1824.16 & 0.14 & T.O. & 19 & 18 & 0 \\
    \textbf{} & \textbf{} & \textbf{max} & 1751.50 & 1824.16 & 0.14 & T.O. & 19 & 18 & 0 \\
    \cline{2-10}
    \textbf{} & \multirow[c]{2}{*}{\textbf{\lp}} & \textbf{min} & 0.02 & 0.01 & 0.02 & 0.08 & 37 & 0 & 0 \\
    \textbf{} & \textbf{} & \textbf{max} & 0.06 & 0.01 & 0.05 & 0.09 & 37 & 0 & 0 \\
    \hline
    \multirow[c]{4}{*}{\textbf{Epidemic19}} & \multirow[c]{2}{*}{\textbf{\qp}} & \textbf{min} & 1753.99 & 1822.90 & 0.16 & T.O. & 20 & 19 & 0 \\
    \textbf{} & \textbf{} & \textbf{max} & 1753.98 & 1822.88 & 0.17 & T.O. & 20 & 19 & 0 \\
    \cline{2-10}
    \textbf{} & \multirow[c]{2}{*}{\textbf{\lp}} & \textbf{min} & 0.03 & 0.02 & 0.02 & 0.09 & 39 & 0 & 0 \\
    \textbf{} & \textbf{} & \textbf{max} & 0.08 & 0.01 & 0.06 & 0.11 & 39 & 0 & 0 \\
    \bottomrule
    \end{tabular}
    \caption{Averaged GrepS and BPIC runtime results in seconds for computing importance.}
    \label{tab:app_epidemic_seperated}
\end{table*}

\end{document}